\documentclass{article}

\usepackage{microtype}
\usepackage{graphicx}
\usepackage{subfigure}
\usepackage{booktabs} % for professional tables

\usepackage{hyperref}

\usepackage[preprint]{neurips_2025}
\usepackage[utf8]{inputenc} % allow utf-8 input
\usepackage[T1]{fontenc}    % use 8-bit T1 fonts
\usepackage{amsmath}
\usepackage{amssymb}
\usepackage{mathtools}
\usepackage{amsthm}
\usepackage{bm}

\usepackage{subfigure}

\usepackage{algorithm}
\usepackage[noend]{algpseudocode}
\usepackage[capitalize,noabbrev]{cleveref}

\theoremstyle{plain}
\newtheorem{theorem}{Theorem}[section]

\newtheorem{lemma}[theorem]{Lemma}
\newtheorem{corollary}[theorem]{Corollary}
\theoremstyle{definition}
\newtheorem{definition}[theorem]{Definition}
\newtheorem{assumption}[theorem]{Assumption}
\theoremstyle{remark}
\newtheorem{remark}[theorem]{Remark}

\usepackage[textsize=tiny]{todonotes}

\title{Convergence Rates of Constrained Expected Improvement}
\author{
Haowei Wang\\
National University of Singapore\\
 Singapore \\
   \texttt{haowei\_wang@u.nus.edu}
   \And
   Jingyi Wang \\
  Lawrence Livermore National Laboratory\\
  Livermore, CA 94550 \\
  \texttt{wang125@llnl.gov} \\
  \And
   Nai-Yuan Chiang \\
  Lawrence Livermore National Laboratory\\
  Livermore, CA 94550 \\
  \texttt{chiang7@llnl.gov} \\
\And
  Zhongxiang Dai\\
  The Chinese University of Hong Kong, Shenzhen\\
China\\
  \texttt{daizhongxiang@cuhk.edu.cn}
  \And
  Szu Hui Ng\\
National University of Singapore\\
Singapore \\
   \texttt{isensh@nus.edu.sg}
   \And
   Cosmin G. Petra \\
  Lawrence Livermore National Laboratory\\
  Livermore, CA 94550 \\
   \texttt{petra1@llnl.gov} \\
}

\begin{document}

\maketitle
%\onecolumn




%\icmlkeywords{Machine Learning, ICML}

\vskip 0.3in
%end two column




\newcommand{\Rbb}{\ensuremath{\mathbb{R} }}
\newcommand{\Nbb}{\ensuremath{\mathbb{N} }}
\newcommand{\Ebb}{\ensuremath{\mathbb{E} }}
\newcommand{\Vbb}{\ensuremath{\mathbb{V} }}
\newcommand{\Cbb}{\ensuremath{\mathbb{C} }}
\newcommand{\Pbb}{\ensuremath{\mathbb{P} }}
\newcommand{\norm}[1]{\left\lVert {#1} \right\rVert}
\newcommand\xbm{{\ensuremath{\bm{x}}}}
\newcommand\kbm{{\ensuremath{\bm{k}}}}
\newcommand\Kbm{{\ensuremath{\bm{K}}}}
\newcommand\fbm{{\ensuremath{\bm{f}}}}
\newcommand\cbm{{\ensuremath{\bm{c}}}}
\newcommand\gbm{{\ensuremath{\bm{g}}}}
\newcommand\ybm{{\ensuremath{\bm{y}}}}
\newcommand\mbm{{\ensuremath{\bm{m}}}}
\newcommand\Ibm{{\ensuremath{\bm{I}}}}
\newcommand{\frank}[1]{\textcolor{blue}{[#1]}}
\newcommand{\cosmin}[1]{\textcolor{red}{[#1]}}

\begin{abstract}
  Constrained Bayesian optimization (CBO) methods have seen significant success in black-box optimization with constraints.  One of the most commonly used CBO methods is the constrained expected improvement (CEI) algorithm. CEI is a natural extension of expected improvement (EI) when constraints are incorporated. However, the theoretical convergence rate of CEI has not been established. 
  In this work, we study the convergence rate of CEI by analyzing its simple regret upper bound.
  %We consider two sets of standard assumptions on the objective $f$ and constraint functions $c$. 
  First, we show that when the objective function $f$ and constraint function $c$ are assumed to each lie in a  reproducing kernel Hilbert space (RKHS), 
  CEI achieves the convergence rates of $\mathcal{O} \left(t^{-\frac{1}{2}}\log^{\frac{d+1}{2}}(t) \right) \ \text{and }\ \mathcal{O}\left(t^{\frac{-\nu}{2\nu+d}} \log^{\frac{\nu}{2\nu+d}}(t)\right)$ for the commonly used squared exponential and Mat\'{e}rn kernels ($\nu>\frac{1}{2}$), respectively.
  Second, we show that when $f$ is assumed to be sampled from Gaussian processes (GPs), CEI achieves similar convergence rates with a high probability.  
  Numerical experiments are performed to validate the theoretical analysis.  
\end{abstract}

\section{Introduction}\label{se:introduction}
%{\color{red} overall flow: 1. BO. 2. review constrained BO. 3. among these, CEI most commonly used. 4. CEI is a direct extension of EI by incorporating constraints. 5. current theoretical results for EI. 6. key technical challenges from EI to CEI 7. No CEI convergence rate results and we establish first for frequentist. 8. Under Bayesian, simple regret upper bound also established. 9. discussion: whether these bounds are tight? relatioship with classic EI? and relationship with noise-free bandit (vakilli)? relationship with the new Japanese UCB results. 10. what can these regret bounds results be used for? }

Bayesian optimization (BO) is an efficient method for optimizing expensive black-box functions without derivatives. It leverages probabilistic surrogate models, most commonly Gaussian processes (GPs), to balance exploration and exploitation in the search for optimal solutions~\citep{frazier2018}. BO has found widespread success in diverse fields such as structural design~\citep{mathern2021}, machine learning hyperparameter tuning~\citep{wu2019hyperparameter}, additive manufacturing process design~\citep{wang2025finite}, fusion design~\citep{wang2024multifidelity}, etc.

%Bayesian optimization is a derivative-free optimization method that uses surrogate models, often a Gaussian process (GP), for black-box objective functions~\citep{frazier2018}.  
%It has seen huge success in many applications including structural design~\cite{mathern2021}, machine learning~\cite{wu2019hyperparameter}, inertial confinment fusion design~\cite{wang2023multifidelity}, process parameter optimization of additive manufacturing~\cite{wang2025finite}, etc. 

While traditional BO is typically applied to unconstrained settings, many real-world problems involve black-box constraints that must be satisfied. This has motivated growing interest in constrained Bayesian optimization (CBO), where surrogate models are also constructed for constraint functions~\citep{bernardo2011} that are complex and expensive to evaluate, making CBO especially valuable in applications like engineering design~\citep{SONG2024109613} and automated machine learning ~\citep{Branke}.  One of the very key difference between unconstrained and constrained optimization is that the feasible region for constrained optimization problem
consists of the search space where all constraints must be satisfied. A general form of the constrained BO problem is: 
\begin{equation} \label{eqn:opt-prob} \centering \begin{aligned} &\underset{\xbm\in C}{\text{minimize}} & & f(\xbm), \ &\text{subject to} & & c(\xbm) \leq 0, \end{aligned} \end{equation} 
where $f:\Rbb^d\to \Rbb$ is the objective function, and $c:\Rbb^d\to\Rbb^m$ are the constraint functions. Both are defined on a compact input space $C \subset \Rbb^d$. The objective and the constraint functions are both expensive black-box functions, that can only be evaluated through expensive physical or computer experiments. 
%Each of these function evaluations may be affected by noise. However, the large majority of research on CBO focuses on noise-free approaches \citep{CBOreview}. 
Throughout this paper, we consider the noise-free setting for both the objective and the constraints, \textit{i.e.},  the function evaluations are deterministic and the true function values can be observed (see Remark~\ref{remark:noisy} for discussion on the noisy case).  In addition, a single constraint is considered, \textit{i.e.}, $m=1$, for simplicity of presentation. We note that our analysis can be easily extended to multiple constraints (see Remark~\ref{remark:multiconstraint} for details).

Broadly, CBO methods can be categorized into implicit and explicit approaches \citep{CBOreview}. Implicit methods modify standard acquisition functions to incorporate constraints via merit functions or feasibility weights. Explicit methods estimate the feasible region directly and restrict search to this region. Among these, the constrained expected improvement (CEI)~\citep{schonlau1998global,gelbart2014bayesian,gardner2014} stands out as one of the most basic and widely adopted methods. 
%It consists of the well-known EI function ~\cite{jones1998efficient} for unconstrained optimization, multiplied by a factor that estimates the joint probability of feasibility. 
CEI is a natural extension of the well-known expected improvement (EI) function~\citep{jones1998efficient}, where the acquisition function is computed as the product of EI and the probability of feasibility. Thanks to this simple and interpretable formulation, CEI has been successfully applied across domains, and it remains one of the default choices in many constrained BO software packages ~\citep{balandat2020botorch}.

Despite its empirical popularity, the theoretical understanding of CEI lags behind. In contrast, unconstrained EI has been more extensively studied. Under a frequentist assumption where the objective  $f$ lies in a reproducing kernel Hilbert space (RKHS), \citet{bull2011convergence} established the convergence rate of EI by deriving the simple regret upper bound. Other works explored the density of sampled sequences~\citep{vazquez2010convergence} or connections between EI and optimal computing budget allocation~\citep{ryzhov2016convergence}. However, convergence rates (\textit{i.e.}, simple regret upper bound) for CEI have not been rigorously established—neither under frequentist nor under Bayesian settings. Here, Bayesian setting means the objective $f$ is a function sampled from a GP.

Introducing constraints into EI significantly complicates the theoretical analysis. Unlike in the unconstrained case, the algorithm may need to explore infeasible regions to gain information on the constraint boundary. Furthermore, CEI’s acquisition function is inherently more complex and non-convex, posing challenges for analysis. On the other hand, the presence of constraints in CEI leads to changes in the sampling procedure. As a result, the key challenge to study the convergence rate of CEI lies in analyzing the exploration (searching for feasible regions) and exploitation (optimizing within feasible areas) since the feasibility threshold is unknown in the input space.

%\frank{We don't need bold letters anymore.}

In this paper, we provide the first theoretical convergence rates for CEI, focusing on simple regret upper bounds under both the frequentist and Bayesian settings. 
Our convergence rates provide practitioners theoretical assurance for the practical deployment of CEI. %Further, we would like to highlight that the obtained bounds under both settings are not direct extensions from current unconstrained results. 
We explain the technical challenges and how we address them in Section~\ref{se:convergence}. Our  contributions are summarized as follows:
\begin{itemize}
    \item Under the frequentist setting, we derive simple regret upper bounds of $\mathcal{O} \left(t^{-\frac{1}{2}}\log^{\frac{d+1}{2}}(t) \right)$ for the squared exponential (SE) kernel and $\mathcal{O}\left(t^{\frac{-\nu}{2\nu+d}} \log^{\frac{\nu}{2\nu+d}}(t)\right)$ for Matérn kernels ($\nu>\frac{1}{2}$). These bounds are improved upon the direct extension of \citet{bull2011convergence} to the constrained case for SE kernel with $d\geq 3$  and Matérn kernels  with $d\geq 3, \nu\geq \frac{d}{d-2}$. (see Theorem \ref{prop:EI-convg-rate-nonoise-gamma}).
    \item  Under the Bayesian setting for the objective, we achieve similar simple regret upper bounds with high probabilities. These bounds are established based on the newly derived bounds (see Theorem \ref{prop:EI-convg-rate-gamma}) on the difference between the improvement function and its corresponding EI in the Bayesian setting.
\end{itemize}

This paper is organized as follows. In Section~\ref{se:background}, we describe the basics and preliminaries of BO, including the CEI algorithm. 
In Section~\ref{se:convergence}, the simple regret upper bounds of CEI are established in both settings.
Numerical experiments to validate the theoretical results are given in Section~\ref{se:exp}.
Conclusions are made in Section~\ref{se:conclusion}. All proof details are presented in the appendix.

\section{Background}\label{se:background}
CBO mainly consists of two components: the GP surrogates for the black-box objective function $f$ and constraint function $c$, and the constrained acquisition function as the sequential sampling rule guiding for the global optimum. 
\subsection{Gaussian process models for $f$ and $c$}  
%{\color{red} Please change everything to the noise-free setting including the GP model, ... expected improvement... Please highlight the GP models for both the objective $f$ and also the constraint function $c$.}
%In the unconstrained setting, BO solves the problem: $\underset{\substack{\xbm}\in C}{\text{minimize}}  f(\xbm)$.
Without losing generality, let the mean function for the objective GP model prior be $0$ and the covariance function (kernel) be  $k_f(\xbm,\xbm'):\Rbb^n\times\Rbb^n\to\Rbb$.
At sample point $\xbm_t\in C$, we denote the objective function value as $f(\xbm_t)$ and the observed constraint function value is $c(\xbm_t)$. 
%, which assumes the prior distribution for $f(\xbm)$ to be Gaussian.
Given $t$ sample points, denote $\xbm_{1:t}=[\xbm_1,\dots,\xbm_t]$ and $\fbm_{1:t}=[f(\xbm_1),\dots,f(\xbm_t)]$. 
Moreover, denote the $t\times t$ covariance matrix $\Kbm_t^f = [k_f(\xbm_1,\xbm_1),\dots,k_f(\xbm_1,\xbm_t); \dots; k_f(\xbm_t,\xbm_1),\dots,k_f(\xbm_t,\xbm_t)]$.
The posterior distribution of  $f(\xbm) | \xbm_{1:t},\fbm_{1:t} \sim \mathcal{N} (\mu_t^f(\xbm),(\sigma_t^f(\xbm))^2)$ can then be inferred using Bayes' rule as follows 
\begin{equation} \label{eqn:GP-post}
 \centering
  \begin{aligned}
  &\mu_t^f(\xbm)\ =\ (\kbm_t^f(\xbm) )^T(\Kbm_t^f)^{-1} \fbm_{1:t},\\
  &(\sigma_t^f)^2(\xbm)\ =\
k_f(\xbm,\xbm)-(\kbm_t^f(\xbm))^T(\Kbm_t^f)^{-1}\kbm_t^f(\xbm)\ ,
\end{aligned}
\end{equation}
where $\kbm_t^f(\xbm)= [k_f(\xbm_1,\xbm),\dots,k_f(\xbm_t,\xbm)]^T$.
Similarly, denote the kernel for $c$ as $k_c:\Rbb^n\times\Rbb^n\to\Rbb$ and the covariance matrix $\Kbm_t^c = [k_c(\xbm_1,\xbm_1),\dots,k_c(\xbm_1,\xbm_t); \dots; k_c(\xbm_t,\xbm_1),\dots,k_c(\xbm_t,\xbm_t)]$. The posterior distribution for $c$ is 
\begin{equation*} \label{eqn:GP-post-c}
 \centering
  \begin{aligned}
  &\mu_t^c(\xbm)\ =\ (\kbm_t^c(\xbm))^T (\Kbm_t^c)^{-1} \cbm_{1:t},\\
  &(\sigma_t^c)^2(\xbm)\ =\
k_c(\xbm,\xbm)-(\kbm_t^c(\xbm))^T(\Kbm_t^c)^{-1}\kbm_t^c(\xbm)\ ,
\end{aligned}
\end{equation*}
where $\kbm_t^c(\xbm)= [k_c(\xbm_1,\xbm),\dots,k_c(\xbm_t,\xbm)]^T$, and $\mu^c_t(\xbm)$ and $(\sigma^c_t)^2(\xbm)$ are the posterior mean and variance for $c$, respectively. Here we use the subscripts $_f$, $_c$ and superscripts $^f$, $^c$ to distinguish between GPs for $f$ and $c$. Choices of the kernels $k_f$ and $k_c$ include the SE and Matérn kernels, which are among the most popular kernels for GP and BO. Their definitions are as follows. 
\begin{equation*} \label{def:sematern}
  \centering
  \begin{aligned}
    k_{SE}(\xbm,\xbm') = \exp\left(-\frac{r^2}{2l^2}\right), \ k_{Mat\acute{e}rn}(\xbm,\xbm')=\frac{1}{\Gamma
    (\nu)2^{\nu-1}}\left(\frac{\sqrt{2\nu}r}{l}\right)^{\nu} B_{\nu} \left(\frac{\sqrt{2\nu}r}{l}\right),
  \end{aligned}
\end{equation*}
where $l>0$ is the length hyper-parameters, $r=\norm{\xbm-\xbm'}_2$, $\nu>0$ is the smoothness parameter of the Matérn kernel, and $B_{\nu}$ is the modified Bessel function of the second kind. 
%For simplicity of analysis, we assume the kernels $k_f$ and $k_c$ for both $f$ and $c$ satisfy $k_f(\xbm,\xbm')\leq 1$, $k_c(\xbm,\xbm')\leq 1$, $k_f(\xbm,\xbm)=1$, and $k_c(\xbm,\xbm)=1$. 
%The posterior predictions~\eqref{eqn:GP-post} applies to the noise-free case by using $\sigma=0$.
\subsection{Constrained Expected Improvement} \label{sec:boalgo}
Acquisition functions are critical to the performances of BO algorithms.
In the unconstrained setting, one of the most widely adopted acquisition functions is EI~\citep{jones1998efficient}. 
Given $t$ samples, the improvement function of $f$ used in EI is defined as
\begin{equation} \label{eqn:improvement}
	\centering
	\begin{aligned}
		I^f_{t}(\xbm) = \max\{ f^+_{t} -f(\xbm),0  \}, 
	\end{aligned}
\end{equation}
where $f^+_t = \underset{\substack{i=1,\dots,t}}{\text{min}} \ f(\xbm_i)$.
The expectation of~\eqref{eqn:improvement} conditioned on existing samples is EI, which has a closed form~\cite{brochu2010}:
\begin{equation} \label{eqn:EI-1}
	\centering
	\begin{aligned}
		EI^f_{t}(\xbm) =   (f^+_{t}-\mu_{t}^f(\xbm))\Phi(z^f_{t}(\xbm))+\sigma_{t}^f(\xbm)\phi(z^f_{t}(\xbm)), 
	\end{aligned}
\end{equation}
where 
$ z^f_{t}(\xbm) = \frac{f^+_{t}- \mu^f_{t}(\xbm)}{\sigma_t^f(\xbm)}$.
The functions 
$\phi$ and~$\Phi$ are the probability density function (PDF)
and the cumulative distribution function (CDF) of the standard normal distribution, respectively.
%We define a useful function in analyzing EI: $\tau(z) = z\Phi(z)+\phi(z)$ for $z\in\Rbb$. Therefore, $EI_t^f(\xbm) = \sigma_t^f(\xbm)\tau(z_t^f(\xbm))$. 
The $t+1$th sample using EI is chosen by  
\begin{equation} \label{eqn:EI-t+1}
	\centering
	\begin{aligned}
		\xbm_{t+1} = \underset{\substack{\xbm\in C}}{\text{argmax}}  EI_{t}^f (\xbm).
	\end{aligned}
\end{equation}

Taking into account the constraint,  the constrained improvement function in CEI~\citep{gardner2014} is defined as  
\begin{equation} \label{eqn:cei-1}
	\centering
	\begin{aligned}
		I^C_{t} = \Delta^c_{t}(\xbm) \max\{f^+_{t}-f(\xbm),0\},
	\end{aligned}
\end{equation}
where $\Delta^c_t \in\{0,1\}$ is the feasibility indicator function where $\Delta^c_t(\xbm)=1$ if $c(\xbm)\leq 0$ and $\Delta^c_t(\xbm)=0$ otherwise.
The incumbent $f_t^+$ in CEI is augmented to be the best feasible observation.
CEI assumes that $f$ and $c$ are conditionally independent~\citep{gardner2014}.
Taking the conditional expectation of~\eqref{eqn:cei-1}, the CEI function is  
\begin{equation} \label{eqn:cei-2}
	\centering
	\begin{aligned}
		EI^C_{t}(\xbm) = P_t(\xbm) EI_{t}^f(\xbm) = \Phi\left(-\frac{\mu_{t}^c(\xbm)}{\sigma_{t}^c(\xbm)}\right) EI_{t}^f(\xbm),
	\end{aligned}
\end{equation}
where $P_t(\cdot)$ is the probability of feasibility (POF) function for $c(\xbm)\leq 0$. CEI chooses the next sample via

\begin{equation} \label{eqn:eci-1}
	\centering
	\begin{aligned}
		\xbm_{t+1} = \underset{\substack{\xbm\in C}}{\text{argmax}} P_t(\xbm) EI_{t}^f(\xbm).
	\end{aligned}
\end{equation}

The  CEI  algorithm is given in Algorithm~\ref{alg:cucb}.  
\begin{algorithm}[H]
 \caption{CEI algorithm}\label{alg:cucb}
  \begin{algorithmic}[1]
          \State{Choose  $k_f(\cdot,\cdot)$, $k_c(\cdot,\cdot)$,  and $T_0$ initial samples $\xbm_i, i=1,\dots,T_0$. Observe $\fbm_{1:T_0}$ and $\cbm_{1:T_0}$.}
         \State{Train the GP surrogate models for $f$ and $c$ respectively conditioned on the initial observations.}
  \For{$t=T_0+1,T_0+2,\dots$}
          \State{Find $\xbm_{t+1}$ based on~\eqref{eqn:eci-1} (CEI).}
          \State{Observe $f(\xbm_{t+1})$ and $c(\xbm_{t+1})$. \;}
          \State{Update the GP models with the addition of $\xbm_{t+1}$, $f(\xbm_{t+1})$, and $c(\xbm_{t+1})$.\;}
          \If {Evaluation budget exhausted}
              \State{Exit}
          \EndIf
  \EndFor
  \end{algorithmic}
\end{algorithm}
CEI can be extended to multiple constraints assuming conditional independence among the constraints~\citep{gardner2014}. Our derived convergence rates can also be readily extended to multiple constraints, as we explain in Remark~\ref{remark:multiconstraint}.

\section{Convergence rates of CEI}\label{se:convergence}
We present our main results of convergence rates for CEI by establishing the simple regret upper bounds.
Denote the optimal solution to the constrained optimization problem~\eqref{eqn:opt-prob} as $\xbm^*$.
In the unconstrained case, the simple regret of EI is defined as $f^+_t-f(\xbm^*)$~\citep{bull2011convergence}. 
In the constrained case, we use the current best feasible observation and compare it to the optimal solution $f(\xbm^*)$, since one could have an infeasible sample point with smaller objective than $f(\xbm^*)$. 
Given that $f^+_t$ is already defined as the best feasible observation till iteration $t$ in CEI, we continue to use 
\begin{equation} \label{def:error-constrained}
	\centering
	\begin{aligned}
		r_t = f^+_t-f(\xbm^*),
	\end{aligned}
\end{equation}
as the simple regret for CEI.  In our analysis, we make the same underlying assumption as CEI that $f_t^+$ exists. 
%\begin{remark}[Infeasible initial samples]
%It is well-known that one potential issue with CEI is the requirement of the initial feasible sample~\citep{gardner2014}. That is, $f_t^+$ exists from the initial samples so that the CEI calculation can proceed. Recognizing this issue, solutions have been proposed including effective heuristics when $f_t^+$ does not exist~\citep{lederer2019uniform} and using the constraint violation function to first find a feasible sample~\citep{jiao2019complete}. In our analysis, we make the same underlying assumption as CEI that $f_t^+$ exists. 
%\end{remark}
In the following, we first establish the convergence rate under the frequentist assumptions in Section~\ref{se:cei}, including an improved version of the rate under frequentist assumptions in Section~\ref{se:improvedrate}. Then, we establish the convergence rate under Bayesian objective  assumptions in Section~\ref{se:cei-bayesian}. 
%{\color{red} i add the reason why we have both frequentist and Bayesian results. please check this. in fact, frequentist is called the agnostic setting in gp-ucb to see when prior of $f$ is misspecified, after the Bayesian setting analysis. should we also switch the sequence to establish bayesian first? i am not quite sure about this}

\subsection{Simple regret upper bound under frequentist assumptions}\label{se:cei}
%{\color{red}  We can list some of the very key/novel results to the main body because this work is essentially theoretical. Highlight the technical challenge also. }

In this section, we present the simple regret upper bound for CEI under the frequentist setting. Moreover, by adopting the information theory-based bounds and techniques in the noise-free cumulative regret bound of upper confidence bound (UCB)~\citep{lyu2019efficient}, we can derive an improved upper bound in some cases compared to~\citet{bull2011convergence}.  
The definition of RKHS is given below.
\begin{definition}\label{def:rkhs}
   Let $k$ be a positive definite kernel $k: \mathcal{X}\times \mathcal{X}\to\Rbb $ with respect to a finite Borel measure supported on $\mathcal{X}$. A Hilbert space $H_k$ of functions on $\mathcal{X}$ with an inner product $\langle \cdot,\cdot \rangle_{H_k}$ is called a RKHS with kernel $k$ if $k(\cdot,\xbm)\in H_k$ for all $\xbm\in \mathcal{X}$, and $\langle f,k(\cdot,\xbm)\rangle_{H_k}=f(\xbm)$ for all $\xbm\in \mathcal{X}, f\in H_k$. The induced RKHS norm $\norm{f}_{H_k}=\sqrt{\langle f,f\rangle_{H_k}}$ measures the smoothness of $f$ with respect to $k$.
\end{definition}

%given the noisy asymptotic convergence of EI itself has not been fully established.
In this section, we assume the following assumptions on the functions $f$ and $c$. 
\begin{assumption}\label{assp:rkhs}
	The functions $f$ and $c$ lie in the RKHS, denoted as $\mathcal{H}_k^f(C)$ and $\mathcal{H}_k^c(C)$ associated with their respective bounded kernel $k_f$ and $k_c$, with the norm $\norm{\cdot}_{H_k^f}$ and $\norm{\cdot}_{H_k^c}$. 
         The kernels satisfy $k_f(\xbm,\xbm')\leq 1$, $k_c(\xbm,\xbm')\leq 1$, $k_f(\xbm,\xbm)= 1$, and $k_c(\xbm,\xbm)= 1$, for $\forall \xbm,\xbm'\in C$.
	The RKHS norms are bounded above by constants $B_f$ and $B_c$, respectively, \textit{i.e.}, $\norm{f}_{H_k^f} \leq B_f$, $\norm{c}_{H_k^c}\leq B_c$. 
    Moreover, the bound constraints set $C$ is compact.
\end{assumption}

\paragraph{Technical Challenges under Frequentist Assumptions.}  The main challenge in establishing a simple regret upper bound for CEI is how to incorporate the  constraint $c(\xbm)\leq 0$ and the probability of feasibility function $P_t(\xbm)$ into the analysis. 
%since they only appear in the constrained setting and therefore the unconstrained analysis techniques do not suffice. 
Existing regret bounds analysis on CBO methods often focus on UCB-type methods~\citep{lu2022no,zhou2022kernelized}, for which the acquisition functions do not have the multiplicative structure between the objective and the constraint. 
%Thus, the analysis of CEI requires novel theoretical techniques. 

Under Assumption~\ref{assp:rkhs}, both $f$ and $c$ are bounded on $C$ by their RKHS norm bounds, as stated in Lemma~\ref{lem:f-bound}. 
The simple regret upper bound is given in the following theorem.
\begin{theorem}\label{theorem:CEI-convg-nonoise}
	Under Assumption~\ref{assp:rkhs}, the CEI algorithm leads to the simple regret upper bound of 
	\begin{equation} \label{eqn:cei-convg-nonoise-1}
		\centering
		\begin{aligned}
			r_t  \leq  \frac{c_{\tau B}}{\Phi(-B_c)} \left[B_f \frac{4}{t-2}+(0.4+B_f) \sigma_{t_k}^f(\xbm_{t_k+1})\right], 
		\end{aligned}
	\end{equation} 
	for some $t_k\in [\frac{t}{2}-1,t]$, and $c_{\tau B}=\frac{\tau(B_f)}{\tau(-B_f)}$.   
\end{theorem}
\paragraph{Sketch of Proof for Theorem~\ref{theorem:CEI-convg-nonoise}.} We start by noticing that the sum of the difference between consecutive best feasible observations is bounded, \textit{i.e.}, $\sum_{t=1}^T f^+_{t-1}-f^+_t \leq 2B_f$. 
Then, we adopt a technique in~\citet{bull2011convergence} to find $t_k$ such that $f^+_{t_k}-f^+_{t_k+1}\leq \frac{2B_f}{k}$, where $k\leq t_k \leq2k$ and $2k\leq t\leq 2(k+1)$. 
Next, using the monotonicity of $f^+_t$, $r_t$ is bounded by $r_{t_k}$. Using the inequality between $I_t^f$ and $EI_t^f$ in Lemma~\ref{lem:EI-ratio-bound}, we can bound $r_{t_k}$ by the EI on objective: $EI_{t_k}^f(\xbm^*)$. 
Then, we transform $EI_{t_k}^f(\xbm^*)$ into $EI_{t_k}^f(\xbm_{t_k+1})$ by inserting the term $P_{t_k}(\xbm^*)$, taking advantage of the multiplicative structure of CEI.  The upper bound of $r_t$ then consists of the term $\frac{1}{P_{t_k}(\xbm^*)}EI_{t_k}^f(\xbm_{t_k+1})$.
From the confidence interval $|f(\xbm)-\mu_t^f(\xbm)|$ (Lemma~\ref{lem:rkhs-bound-nonoise})  and the fact that $f^+_{t_k}-f^+_{t_k+1}\leq \frac{2B_f}{k}$, we can bound $EI_{t_k}^f(\xbm_{t_k+1})$. 
The constraint term $\frac{1}{P_{t_k}(\xbm^*)}$ remains to be bounded. We use the confidence interval on $|c(\xbm)-\mu_t^c(\xbm)|$ in Lemma~\ref{lem:rkhs-bound-nonoise} at $\xbm^*$ and the fact that $\xbm^*$ is a feasible solution to obtain a lower bound for $P_{t_k}(\xbm^*)$. This concludes the proof. 

\begin{remark}[Constraint in the simple regret upper bound]
    The terms derived from the constraint function in~\eqref{eqn:cei-convg-nonoise-1} is $ \frac{1}{\Phi(-B_c)}$, which emerges from the probability of feasibility function and $\mu_t^c(\xbm)$ and  $\sigma_t^c(\xbm)$ of the GP model of $c(\xbm)$. Thanks to the multiplicative structure between the objective and constraint in $I_t^C$~\eqref{eqn:cei-1} and $EI_t^c$~\eqref{eqn:cei-2}, the simple regret upper bound maintains a similar form. 
\end{remark}

It is clear from~\eqref{eqn:cei-convg-nonoise-1} that the convergence of $r_t$ relies on the posterior standard deviation $\sigma_t^f(\xbm_{t+1})$.
Since $t_k$ increases with $t$, as $\sigma_t^f(\xbm_{t+1})\to 0$, so does $\sigma_{t_k}^f(\xbm_{t_k+1})$.
In the noise-free setting, the posterior variance can be bounded via the maximum distance between sample points and a given point. 
To obtain the rate of simple regret bound, we use Assumptions (1)-(4) in~\citet{bull2011convergence} 
and focus on squared exponential (SE) and Matérn kernels.
Recall that the smoothness parameter of the Matérn kernel is  $\nu>0$.
Both the SE and Matérn kernels satisfy Assumptions (1)-(4) in~\cite{bull2011convergence}, with SE kernel obtained as $\nu\to\infty$.
Further, define 
\begin{equation} \label{eqn:def-nu}
	\centering
	\begin{aligned}
		\eta=\begin{cases} \alpha, \ \nu\leq 1\\
			0, \ \nu>1,
		\end{cases}
	\end{aligned}
\end{equation}
where $\alpha=\frac{1}{2}$ if $\nu\in\Nbb$, and $\alpha=0$ otherwise.
Then, for SE and Matérn kernels, $\sigma_{t_k}^f(\xbm_{t_k+1})$ can be bounded with the following lemma.
\begin{lemma}[\citet{bull2011convergence}]\label{lem:EI-sigma}
  For the SE kernel, there exists constant $C'>0$ so that given $\forall t\in\Nbb$,   
  \begin{equation} \label{eqn:EI-sigma-1}
  \centering
  \begin{aligned}
         \sigma_i^f(\xbm_{i+1}) \geq C' k^{-\frac{1}{d}}
  \end{aligned}
  \end{equation}
  holds for at most $k$ times,  for $\forall k\in\Nbb$, $k\leq t$ and $i=1,\dots,t-1$.
  For Matérn kernels,
   \begin{equation} \label{eqn:EI-sigma-2}
  \centering
  \begin{aligned}
         \sigma_i^f(\xbm_{i+1}) \geq C' k^{-\frac{\min\{\nu,1\}}{d}} \log^{\eta}(k)
  \end{aligned}
  \end{equation}
  holds at most $k$ times.
\end{lemma}
In the constrained setting, we are able to obtain the same rates as those in the unconstrained case~\citep{bull2011convergence} using Lemma~\ref{lem:EI-sigma}.  

\begin{corollary}\label{prop:EI-convg-rate-nonoise}
	Under Assumption~\ref{assp:rkhs}, the CEI algorithm leads to the convergence rates of  
	\begin{equation} \label{eqn:CEI-convg-rate-1}
		\centering
		\begin{aligned}
			\mathcal{O} \left(t^{-\frac{1}{d}} \right) \ \text{and }\ \mathcal{O}\left(t^{-\frac{\min\{\nu,1\}}{d}} \log^{\eta}(t)\right),
		\end{aligned}
	\end{equation}
    for SE and Matérn kernels, respectively, where $\eta$ is from~\eqref{eqn:def-nu}. 
\end{corollary}
Corollary \ref{prop:EI-convg-rate-nonoise} shows that the CEI algorithm is guaranteed to find the best feasible point asymptotically with the rates elaborated in \eqref{eqn:CEI-convg-rate-1}.
Also, we point out that the choice of kernels and their parameters affect the convergence rates. 
Since the SE kernel can be viewed as a Matérn kernel with $\nu\to\infty$, its convergence rate is better than Matérn kernels with $\nu\leq 1$. However, due to the limitations of the kernel analysis in~\cite{bull2011convergence} (see Remark~\ref{remark:improvedrate}), for $\nu\geq 1$, SE and Matérn kernels have similar convergence rates in Corollary~\ref{prop:EI-convg-rate-nonoise}. As we present in the following section, improved rates for both kernels can be obtained in some cases.

\subsubsection{Improved simple regret upper bound under frequentist assumptions}\label{se:improvedrate}
Next, we apply maximum information gain and the corresponding information theory to obtain improved simple regret upper bounds. 
\begin{theorem}\label{prop:EI-convg-rate-nonoise-gamma}
	Under Assumption~\ref{assp:rkhs}, the CEI algorithm leads to the improved convergence rates of  
	\begin{equation} \label{eqn:CEI-convg-rate-2}
		\centering
		\begin{aligned}
			\mathcal{O} \left(t^{-\frac{1}{2}}\log^{\frac{d+1}{2}}(t) \right) \ \text{and }\ \mathcal{O}\left(t^{\frac{-\nu}{2\nu+d}} \log^{\frac{\nu}{2\nu+d}}(t)\right),
		\end{aligned}
	\end{equation}
    for SE and Matérn kernels, respectively. 
\end{theorem}
\paragraph{Sketch of Proof for Theorem~\ref{prop:EI-convg-rate-nonoise-gamma}.} The proof follows similar steps to that of Theorem~\ref{theorem:CEI-convg-nonoise} but further bounds $\sigma_{t_k}^f(\xbm_{t_k+1})$ using $\gamma_t^f$.   
To do so, we first recognize that the bound using $\gamma_t^f$ (Lemma~\ref{lem:variancebound}) is established in the noisy case where the posterior standard deviation has a different form as in~\eqref{eqn:GP-post-2}. Using Lemma~\ref{lem:sigma-noise}, we can establish that the noise-free posterior standard deviation also satisfies $\sum_{i=0}^{t-1} \sigma_{i}^f(\xbm_{i+1}) \leq  \sqrt{C_{\gamma} t \gamma_t^f}$. Then, from Lemma~\ref{lem:EI-sigma-gamma}, we can find a small enough $\sigma_{i}^f(\xbm_{i+1})$. Specifically, choose $k=[t/3]$, where $[x]$ denotes the largest integer smaller than $x$. Thus, we have $3k\leq t\leq 3(k+1)$. Then, there exists $k\leq t_k \leq 3k$ such that $f^+_{t_k}-f^+_{t_k+1}\leq \frac{2B_f}{k}$ and $\sigma_{t_k}^f(\xbm_{t_k+1})\leq \frac{\sqrt{t \gamma_t^f}}{k} $.
The rest of the proof follows from that of Theorem~\ref{prop:EI-convg-rate-nonoise-gamma}.
\\

\begin{remark}[Improved rate of convergence]\label{remark:improvedrate}
   As mentioned above, the rates in Corollary \ref{prop:EI-convg-rate-nonoise} are the same as the known convergence rates for EI in~\citet{bull2011convergence}. 
    Meanwhile, the rates in Theorem~\ref{prop:EI-convg-rate-nonoise-gamma} is an improvement over those of~\citet{bull2011convergence} for SE kernel with $d\geq 3$  and Matérn kernels with $d\geq 3, \nu\geq \frac{d}{d-2}$. 
    To achieve this, we applied techniques from regret bound analysis on noise-free UCB~\citep{lyu2019efficient} that allows us to use maximum information gain to bound the sum of  $\sigma_{t}^f(\xbm_{t+1})$. 
    Then, we use our techniques in the proof of Theorem~\ref{theorem:CEI-convg-nonoise} to bound an individual $\sigma_{t_k}^f(\xbm_{t_k+1})$. 
    In~\cite{bull2011convergence}, the $\sigma_{t_k}^f(\xbm_{t_k+1})$ is bounded by the Taylor expansion of the kernel functions. Therefore, the rates of decrease are limited to quadratic terms for both SE and Matérn kernels, since their Taylor expansions around $0$ for $\norm{\xbm-\xbm'}_2$ are quadratic at best. On the other hand, maximum information gain can lead to tight bounds on $\gamma_t^f$ that take advantages of the spectral properties of the kernels~\citep{vakili2021information,iwazaki2025improved}. Hence, using $\gamma_t^f$ to bound $\sigma_{t_k}^f(\xbm_{t_k+1})$ can produce a faster rate. 
    As the open question raised in~\cite{vakili2022open} gets answered, further improvement of the convergence rates is possible, \textit{e.g.}, using techniques from~\cite{iwazaki2025gaussian}.
    %We note that the latest progress in noise-free UCB regret analysis~\citep{} might further improve the simple regret bound for both EI and CEI in the noise-free case. 
\end{remark}

\subsection{Simple regret upper bound under Bayesian objective assumption}\label{se:cei-bayesian}
In this section, we present the simple regret upper bound for CEI under the Bayesian objective assumptions. We again use the maximum information gain to derive the simple regret upper bound. 

\begin{assumption}\label{assp:GP}
 The bound constraint set $C\subset [0,r]^d$ is compact and convex.
	The objective function $f$ is sampled from  $GP(0,k_f(\xbm,\xbm'))$.
    Further, the objective function $f$ is assumed to be Lipschitz continuous (of 1-norm) with Lipschitz constant $L_f$ with probability $\geq 1-da_fe^{L_f^2/b_f^2}$ for some constants $a_f>0$ and $b_f>0$.
    The kernels satisfy $k_f(\xbm,\xbm')\leq 1$ and  $k_f(\xbm,\xbm)= 1$.
    The constraint function $c$ remains in the RKHS of $k_c$, similarly to the frequentist setting. 

\end{assumption}
%\begin{assumption}\label{assp:kbounded}
	%For $\forall \xbm,\xbm'\in C$, there exists constant %$L>0$ such that the following holds 
%	\begin{equation} \label{eqn:kbounded-1}
%		\centering
%		\begin{aligned}
%			 |f(\xbm)-f(\xbm')| \leq  L \norm{\xbm-\xbm'}_1,
%		\end{aligned}
%	\end{equation}
%\end{assumption}
In the remaining of this section we will work under Assumption~\ref{assp:GP}.

\paragraph{Technical Challenges under Bayesian  Assumptions.}
In addition to the challenges in the frequentist setting, the bounds on EI in the Bayesian setting are not available in current literature, to the best of our knowledge. Starting from the confidence interval on $|f(\xbm)-\mu_t^f(\xbm)|$, we derive the bounds on $|I_t^f(\xbm)-EI_t^f(\xbm)|$ with high probability, an important step towards the bound on $r_t$ . Noticeably, under the Bayesian setting, the bounds are satisfied with a given probability, \textit{e.g.}, $1-\delta$, where $\delta \in (0,1)$. 
%given the noisy asymptotic convergence of EI itself has not been fully established.

The simple regret upper bound is given in the following theorem.
\begin{theorem}\label{theorem:CEI-convg-noise}
   Let $\beta=2\log(6c_{\alpha}/\delta)$ and $\beta_t=2\log(3\pi_t/\delta)$, where $c_{\alpha} = \frac{1+2\pi}{2\pi}$ and $\pi_t=\frac{\pi^2t^2}{6}$. Under Assumption~\ref{assp:GP}, the CEI algorithm leads to the simple regret upper bound  
	\begin{equation} \label{eqn:cei-convg-noise-1}
		\centering
		\begin{aligned}
          r_t  \leq&  c_{\tau}(\beta)\frac{1}{\Phi(
          -B_c)}\left[\frac{4M_f}{t-2}+\frac{2\beta^{1/2}_t}{t-2}\sqrt{C_{\gamma}t\gamma_t^f} +(0.4+\beta^{1/2})\sigma_{t_k}^f(\xbm_{t_k+1})\right],\\
		\end{aligned}
	\end{equation} 
	for some $t_k\in [\frac{t}{2}-1,t]$, $c_{\tau}( \beta)=\frac{\tau(\beta^{1/2})}{\tau(-\beta^{1/2})}$, and constant $M_f>0$ with probability $\geq 1-\delta$. 
\end{theorem}
The constant $M_f$ is from Lemma~\ref{lem:f-bound-bayesian}.
The convergence rate is given in the next theorem.
\begin{theorem}\label{prop:EI-convg-rate-gamma}
Let $\beta=2\log(6c_{\alpha}/\delta)$ and $\beta_t=2\log(3\pi_t/\delta)$, where $c_{\alpha} = \frac{1+2\pi}{2\pi}$ and $\pi_t=\frac{\pi^2t^2}{6}$. Under Assumption~\ref{assp:GP}, the CEI algorithm leads to the convergence rates of 
	\begin{equation} \label{eqn:CEI-convg-rate-2}
		\centering
		\begin{aligned}
			\mathcal{O} \left(t^{-\frac{1}{2}}\log^{\frac{d+2}{2}}(t) \right) \ \text{and }\ \mathcal{O}\left(t^{\frac{-\nu}{2\nu+d}} \log^{\frac{2\nu+0.5d}{2\nu+d}}(t)\right),
		\end{aligned}
	\end{equation}
    for SE and Matérn kernels, respectively, with probability $\geq 1-\delta$. 
\end{theorem}
\paragraph{Sketch of Proof for Theorem~\ref{theorem:CEI-convg-noise}.}
Recall that $f$ and $c$ are assumed conditionally independent in CEI. 
We start from the bound on the confidence interval for $f$: $|f(\xbm)-\mu_t^f(\xbm)|\leq \beta^{1/2}\sigma_t^f(\xbm)$, with probability $\geq 1-\delta $, where $\beta=2\log(1/\delta)$, as in Lemma~\ref{lem:fmu}. The confidence interval of $c$ remains the same as in the frequentist setting. 
These are well-known results \citep{srinivas2009gaussian}.
Then, we derive the subsequent bounds $|I_t^f(\xbm)-EI_t^f(\xbm)|\leq \sqrt{\beta}\sigma_t^f(\xbm)$, where $\beta=\max\{1.44,2\log(c_{\alpha}/\delta)\}$ and $c_{\alpha}=\frac{1+2\pi}{2\pi}$ with probability $\geq 1-\delta$ (Lemma~\ref{lem:IEIbound}).
Then, we prove the relationship in Lemma~\ref{lem:IEIbound-ratio-bay} that $ I_t^f(\xbm) \leq \frac{\tau(\sqrt{\beta})}{\tau(-\sqrt{\beta})} EI_{t}^f(\xbm)$ with probability $\geq 1-\delta$. We can now follow the general analysis framework in Section~\ref{se:cei} and Theorem~\ref{theorem:CEI-convg-nonoise} to obtain the simple regret upper bound under Bayesian objective assumptions, while choosing $t_k$ with a more defined criterion.  

\begin{remark}[Comparison to the frequentist setting]
    Comparing Theorem~\ref{prop:EI-convg-rate-nonoise-gamma} to Theorem~\ref{prop:EI-convg-rate-gamma}, the convergence rates in the frequentist and Bayesian settings are the same except for a $\log^{1/2}(t)$ term. This is partially because simple regret focuses on the best feasible solution $f_t^+$ and thus many of the parameters in Theorem~\ref{theorem:CEI-convg-nonoise} and~\ref{theorem:CEI-convg-noise} do not depend on $t$. 
\end{remark}

\begin{remark}[Multiple constraints]\label{remark:multiconstraint}
    As mentioned in Section~\ref{se:introduction}, our results can be readily applied to CEI with multiple constraints for both frequentist and Bayesian settings. Consider $m$ constraints $c_i(\xbm) \leq 0, i=1,\dots,m$. Assuming conditional independence of the constraints, the CEI function is $EI^C_{t}(\xbm) = \Pi_{i=1}^m P_t^i(\xbm) EI_{t}^f(\xbm) = \Pi_{i=1}^m \Phi\left(\frac{-\mu_{t}^{c_i}(\xbm)}{\sigma_{t}^{c_i}(\xbm)}\right) EI_{t}^f(\xbm)$, where $P_t^i$ is the probability of feasibility function of constraint $c_i(\xbm)\leq 0$, and $\mu_{t}^{c_i}(\xbm)$ and $\sigma_{t}^{c_i}(\xbm)$ are the posterior mean and standard deviation for $c_i$, respectively. By making the assumption that each constraint function lies in its corresponding RKHS of the kernel $k_{c_i}$, we have $|c_i(\xbm)-\mu_t^{c_i}(\xbm)|\leq B_{c_i}\sigma_t^{c_i}(\xbm)$, where $B_{c_i}$ is the upper bound of RKHS norm associated with kernel $k_{c_i}$ and function $c_i$. 
    We can then apply the analysis framework in this paper to obtain an upper bound similar to that of Theorem~\ref{theorem:CEI-convg-nonoise}, 
    where the term $\frac{1}{\Phi(
          -B_c)}$ is replaced with  $\Pi_{i=1}^m\frac{1}{\Phi(
         -B_{c_i})}$. We note that in the Bayesian objective setting, to ensure probability $1-\delta$, the parameter $\beta$ needs to increase with the number of constraints as well, \textit{e.g.}, $\beta=2\log((m+5)c_{\alpha}/\delta)$.
\end{remark}

\begin{remark}[Extension to the noisy setting]\label{remark:noisy}
Extending our analysis to the noisy setting is non-trivial, and we discuss the associated challenges for noisy objective and constraint functions separately. A noisy constraint function introduces additional complications in defining feasibility. If only noisy observations of the constraint values are available, the notion of a feasible sample and the definition of $f_t^+$ becomes ambiguous. As a result, major modifications to the CEI algorithm are required to appropriately handle the uncertainty introduced by noise.
    
For the noisy objective function, CEI can be adapted similarly to the noisy EI formulation by treating the best feasible noisy observation as the incumbent. However, to the best of our knowledge, a theoretical guarantee on the simple regret bound for the noisy unconstrained setting remains unavailable. Recent work by \citet{wang2025convergence} provides a framework for deriving noisy simple regret bounds based on the best observed value, $r_t^s = y_t^+ - f(\xbm_t)$, which can be extended to CEI. Specifically, by defining $r_t^s$ as the simple regret for CEI with $y_t^+$ denoting the best feasible noisy observation, a similar proof strategy as in Theorem~\ref{theorem:CEI-convg-noise} yields an analogous upper bound. In the Bayesian setting with i.i.d. Gaussian noise on the objective and noise-free constraint observations, the convergence rate of the upper bound on $r_t^s$ can be obtained. However, we note that given the noise, $r_t^s$ is possibly negative.  
\end{remark}

\begin{remark}[Infeasible initial sample]
It is well known that  CEI requires initial feasible sample~\citep{gardner2014}. That is, $f_t^+$ exists from the initial samples so that the CEI calculation can proceed. %Recognizing this issue, solutions have been proposed including effective heuristics when $f_t^+$ does not exist~\citep{lederer2019uniform} and using the constraint violation function to first find a feasible sample~\citep{jiao2019complete}. 
Methods proposed to address this issue typically employ separate strategies when no feasible samples are available and revert to the standard CEI formulation once feasibility is established~\citep{lin2024multi,letham2019constrained}. In addition, introducing a tolerance parameter in the constraint can further mitigate this problem by allowing near-feasible points when the degree of violation is small.  
\end{remark}

\begin{remark}[Tolerance in constraints]
In gradient-based optimization methods, a tolerance for constraint violation is often used to improve the performance and flexibility of algorithms~\cite{ipopt,optimization}. Motivated by this, we introduce a tolerance parameter $\lambda \ge 0$, where a point $\xbm$ is considered feasible if $c(\xbm) \le \lambda$ and infeasible otherwise. The corresponding CEI with tolerance is defined as
$EI^C_{t}(\xbm,\lambda) = P_t(\xbm,\lambda) EI_{t}^f(\xbm) = \Phi\left(\frac{\lambda-\mu_{t}^c(\xbm)}{\sigma_{t}^c(\xbm)}\right) EI_{t}^f(\xbm).$ Clearly, the standard CEI formulation is recovered when $\lambda = 0$. 

The simple regret bound is affected by $\lambda$ and should lead to $\frac{1}{\Phi(\lambda-B_c)}$ in place of $\frac{1}{\Phi(-B_c)}$. 
In fact, we can replace $\frac{1}{\Phi(\lambda-B_c)}$ with  $1/{\Phi\left(\frac{\lambda}{\sigma^c_{t_k}(x^*)}-B_c\right)}$, which is time-varying. One can follow the proof of Theorem~\ref{theorem:CEI-convg-nonoise} to obtain this term, which emerges from the confidence interval of $c$ at $t_k$, $x^*$ and $c(x^*)\leq0$. 

As the sample iteration increases, the inclusion of $\sigma^c_{t_k}(x^*)$ is important in balancing $-B_c$ that can lead to a large simple regret upper bound.  We explain the intuition below. 
As $t\to\infty$, $t_k\to\infty$ and $k\to\infty$. We know $\sigma_t^f(\xbm_{t+1})\to 0$, and hence $\sigma_{t_k}^f(\xbm_{t_k+1})\to 0$ and $r_t\to 0$. That is, the simple regret upper bound of CEI with $\lambda>0$ converges.  
Thus, $\xbm_t$ approaches at least one of the optimal solutions. Suppose without losing generality, $\xbm_t\to \xbm^*$. Then, by definition $\sigma_{t_k}^c(\xbm^*)\to 0$.
Consequently, we should have $\frac{\lambda}{\sigma^c_{t_k}(x^*)}\rightarrow \infty$ for $\lambda>0$. Then, we have $\Phi\left(\frac{\lambda}{\sigma^c_{t_k}(x^*)}\right)\to 1$. Therefore, $1/{\Phi(\frac{\lambda}{\sigma^c_{t_k}(x^*)}-B_c)}\to 1$ and $B_c$ does not affect the simple regret upper bound asymptotically. We note that the convergence rate of CEI with tolerance remains similar since it is dominated by the maximum information gain of $f$.
\end{remark}

%\cosmin{At this point I have a major question/``frustration'': the regret incorporates only objective function and all convergence (rates) are for it; the question is what about the constraints? I had to dig it up and found some intuition in the proof of Theorem 3.2. But I still find the main theorems are confusing since they tacitly assume that the simple regret bounds are at feasible points; btw, is there any ``guarantee'' that a feasible point is reached? It would not hurt to clarify all these. Also it would help with the understanding if you could compare rates with the unconstrained case.}

%\begin{remark}[Extension to the noisy setting]
%    A natural extension of our result is the simple regret bound of CEI in the noisy case, or at least when the objective contains observation noise. This is challenging as the noisy regret for EI (without constraint) has not been established, to the best of our knowledge. 
%\end{remark}

%{\color{red} discussion 3: briefly discuss and explain how our results can be potentially extended for noisy settings}

%{\color{red} implications. our bounds can be used for... maybe we can refer the ICML submission}
\section{Numerical experiments}\label{se:exp}
Although this paper is primarily theoretical, we conduct numerical experiments to support the theoretical results. We apply the CEI algorithm to eight synthetic problems that are randomly generated from RKHS of kernels and GP priors, and five benchmark problems commonly used in the CBO literature. These numerical experiments are not intended to demonstrate superior performance over the state-of-the-art CBO algorithms. Instead, they serve as empirical evidence for the theoretical analysis presented in this work. All experiments are conducted on M1 (16GB memory)\footnote{Codes are available in https://github.com/Haowei-Wang/Convergence-Rates-of-Constrained-Expected-Improvement.}.

\subsection{Synthetic problems}
In this section, we study objective and constraint functions drawn from reproducing kernel Hilbert spaces (RKHSs) as well as from Gaussian process priors with  Matérn  ($\nu = 2.5$) and squared exponential (SE) kernels, across input dimensions $d \in {2,4}$. The domain is the hypercube $[0,1]^d$. For RKHS cases (the frequentist setting), the functions are generated with a similar approach to~\cite{chowdhury2017kernelized}. Specifically, both objective $f(x)$ and constraint functions $c(x)$ are generated by sampling from the RKHS associated with a chosen kernel (Matérn/SE kernels with a length scale of $0.2$). Each function is constructed as a weighted sum of kernel evaluations at $100$ randomly selected basis points, with weights drawn from a standard normal distribution. Formally, the function takes the form $f(x) = \sum_{i=1}^{n} \alpha_i k(x, X_i)$ , where $k$  is the kernel, $X_i$ are basis points, and $\alpha_i$ are random coefficients; $c(x)$ is generated similarly. For the GP cases (the Bayesian setting),  the functions are generated with an approach similar to~\cite{srinivas2009gaussian}. Specifically, we uniformly choose 1000 points in the design space and sample randomly from a multivariate Gaussian distribution defined by the GP prior with the chosen kernel. 

For each synthetic problem, we conducted 100 independent trials. The number of initial design is set to $10d$, and 50 optimization iterations were performed for all cases. We plotted the log-log curve of simple regret against the number of iterations in Figure \ref{fig:synthetic}.  In all cases, we consistently observed sublinear convergence patterns, which align well with our theoretical guarantees.  

\begin{figure}[!h]
\subfigure[RKHS, Matérn, $d=2$]{
	\centering
	\includegraphics[width=0.225\textwidth]{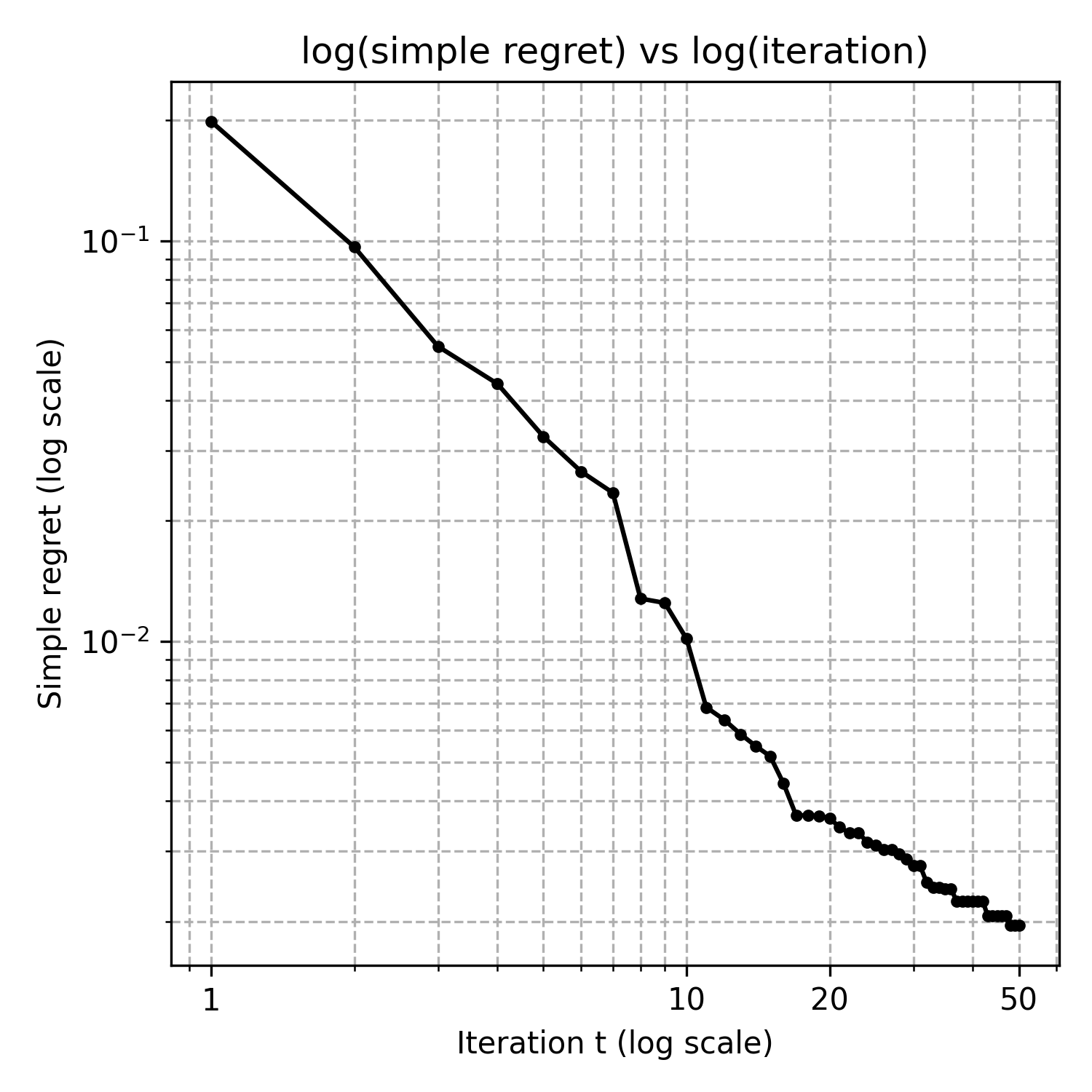}
	\label{fig:s1}
}
\subfigure[RKHS, Matérn, $d=4$]{
	\centering
	\includegraphics[width=0.225\textwidth]{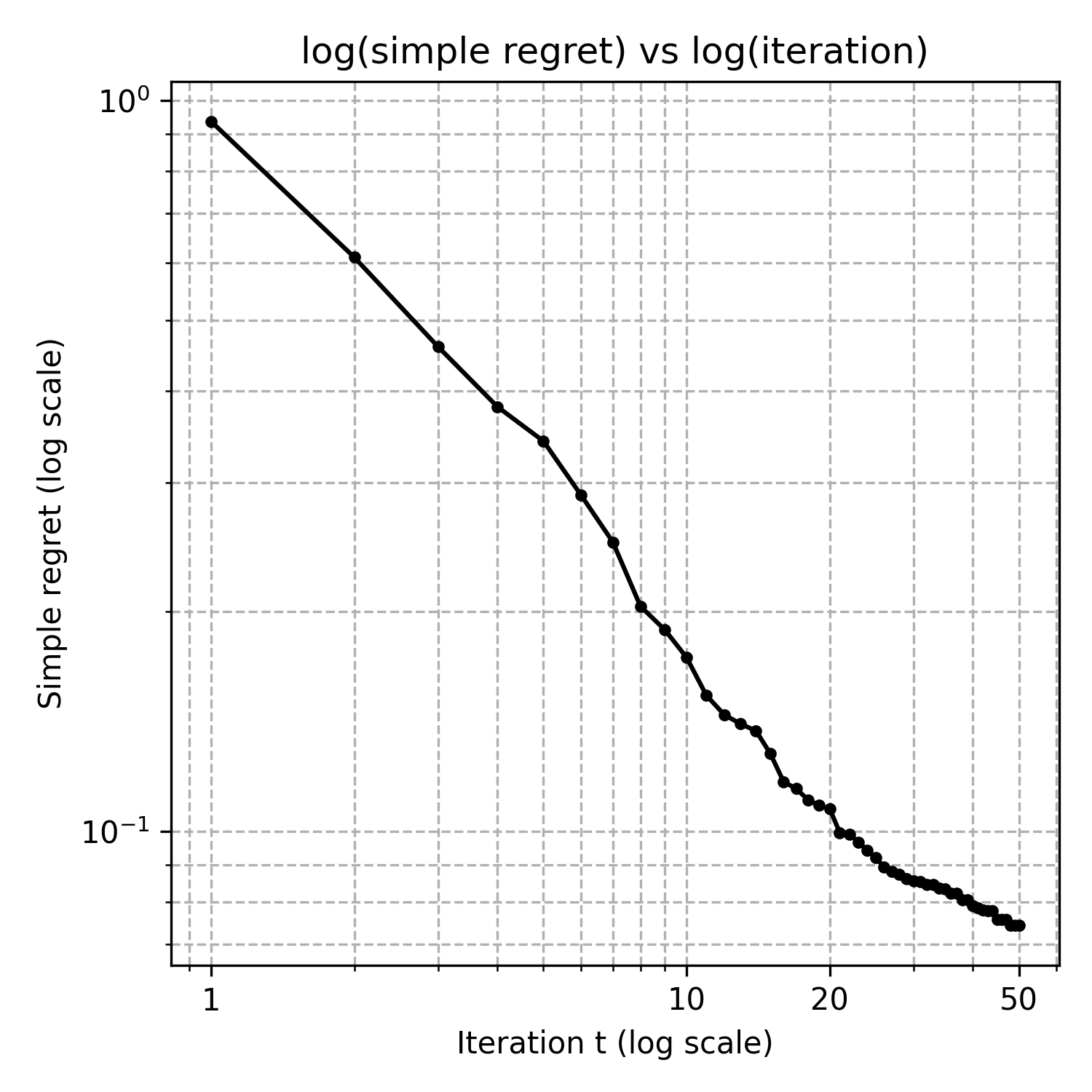}
	\label{fig:s2}
	}
\subfigure[RKHS, SE, $d=2$]{
	\centering
	\includegraphics[width=0.225\textwidth]{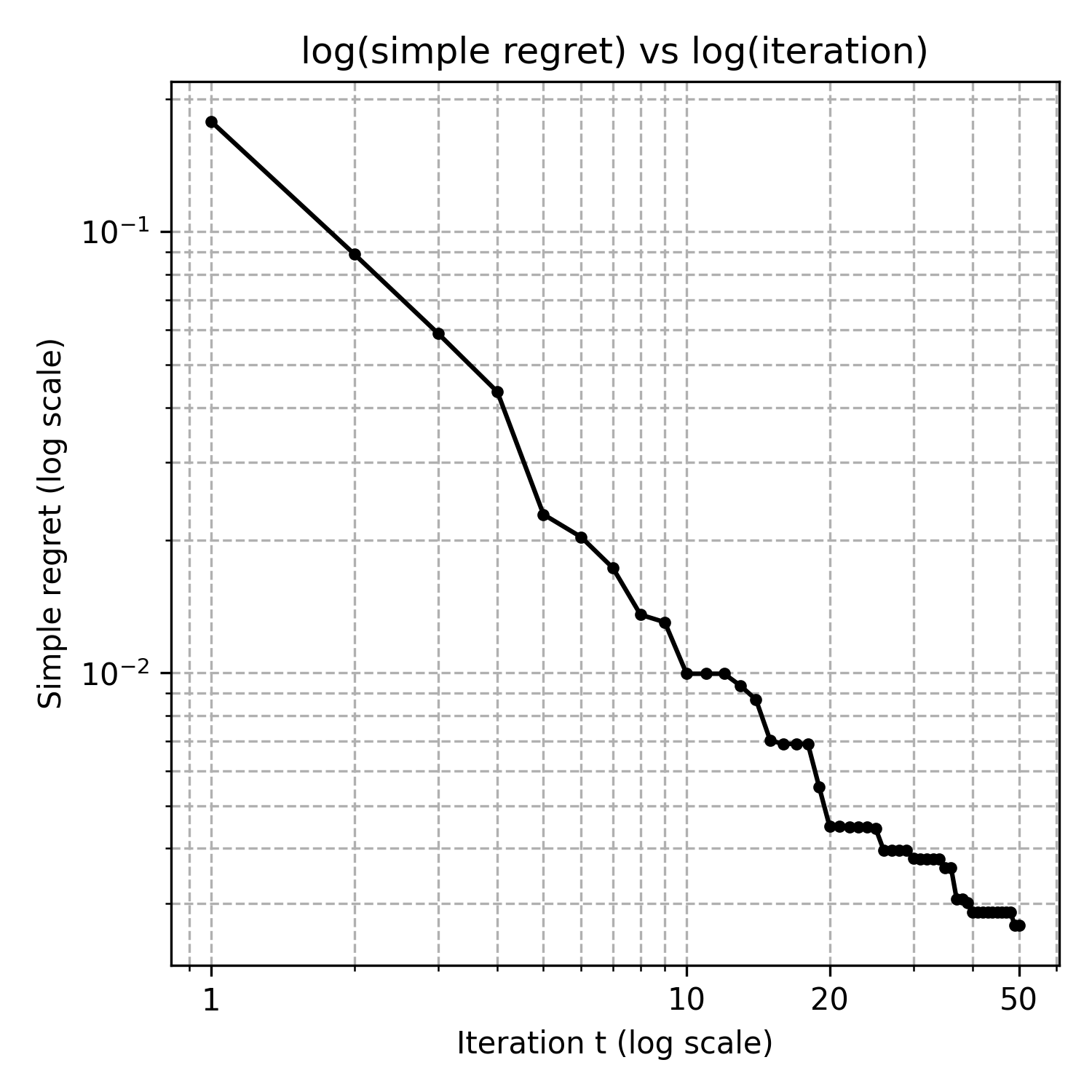}
	\label{fig:3}
}
\subfigure[RKHS, SE, $d=4$]{
	\centering
	\includegraphics[width=0.225\textwidth]{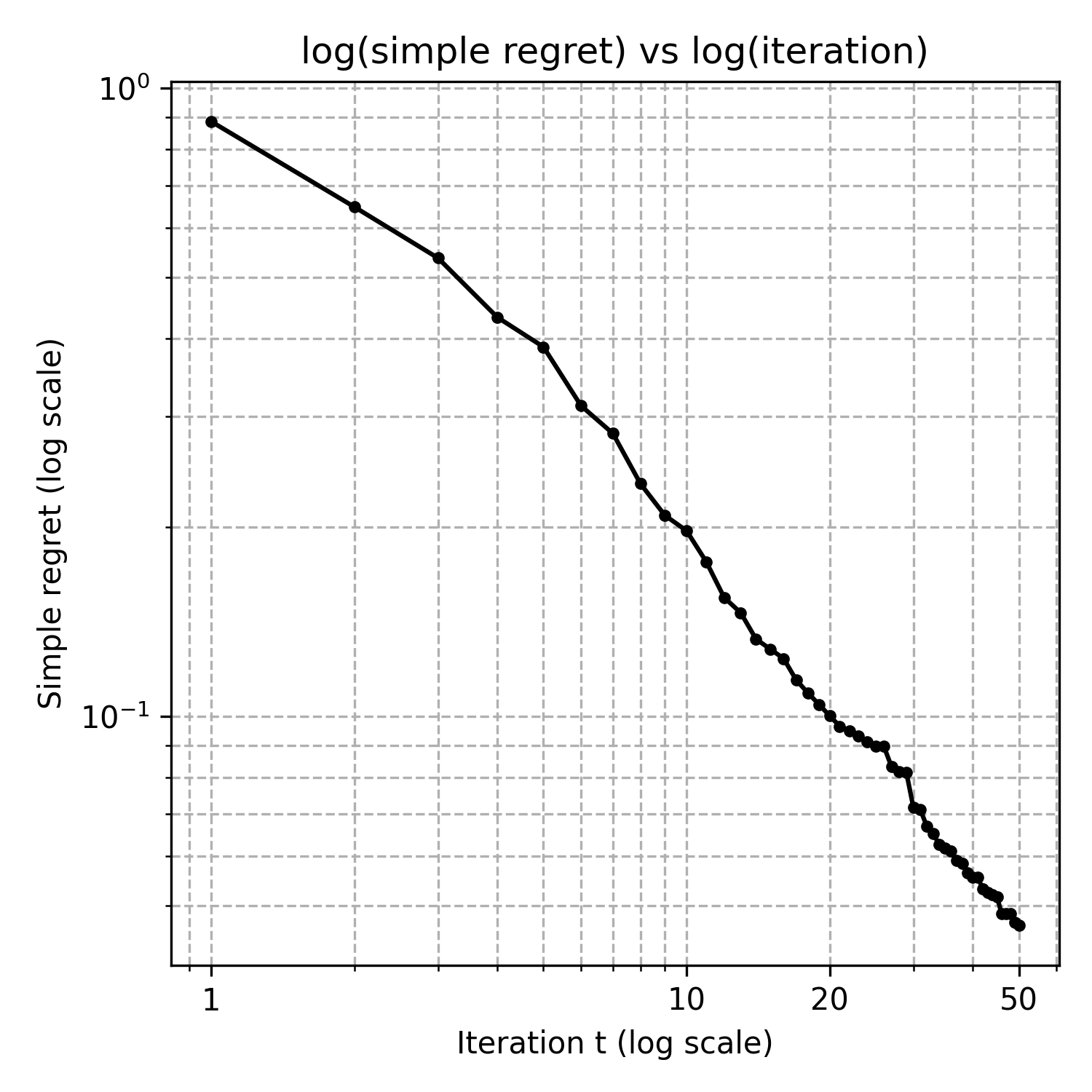}
	\label{fig:4}
}
\subfigure[GP, Matérn, $d=2$]{
	\centering
	\includegraphics[width=0.225\textwidth]{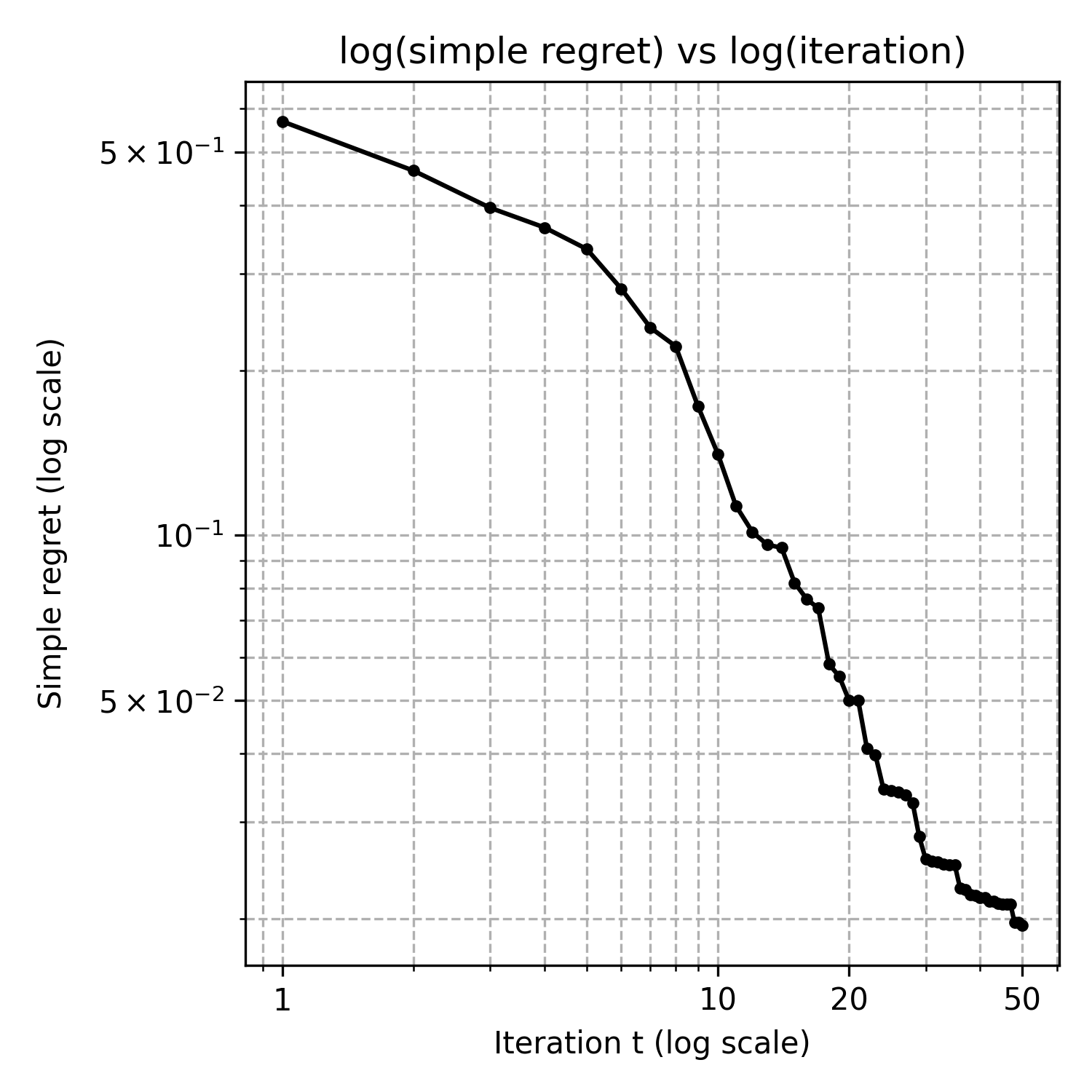}
	\label{fig:5}
}
\subfigure[GP, Matérn, $d=4$]{
	\centering
	\includegraphics[width=0.225\textwidth]{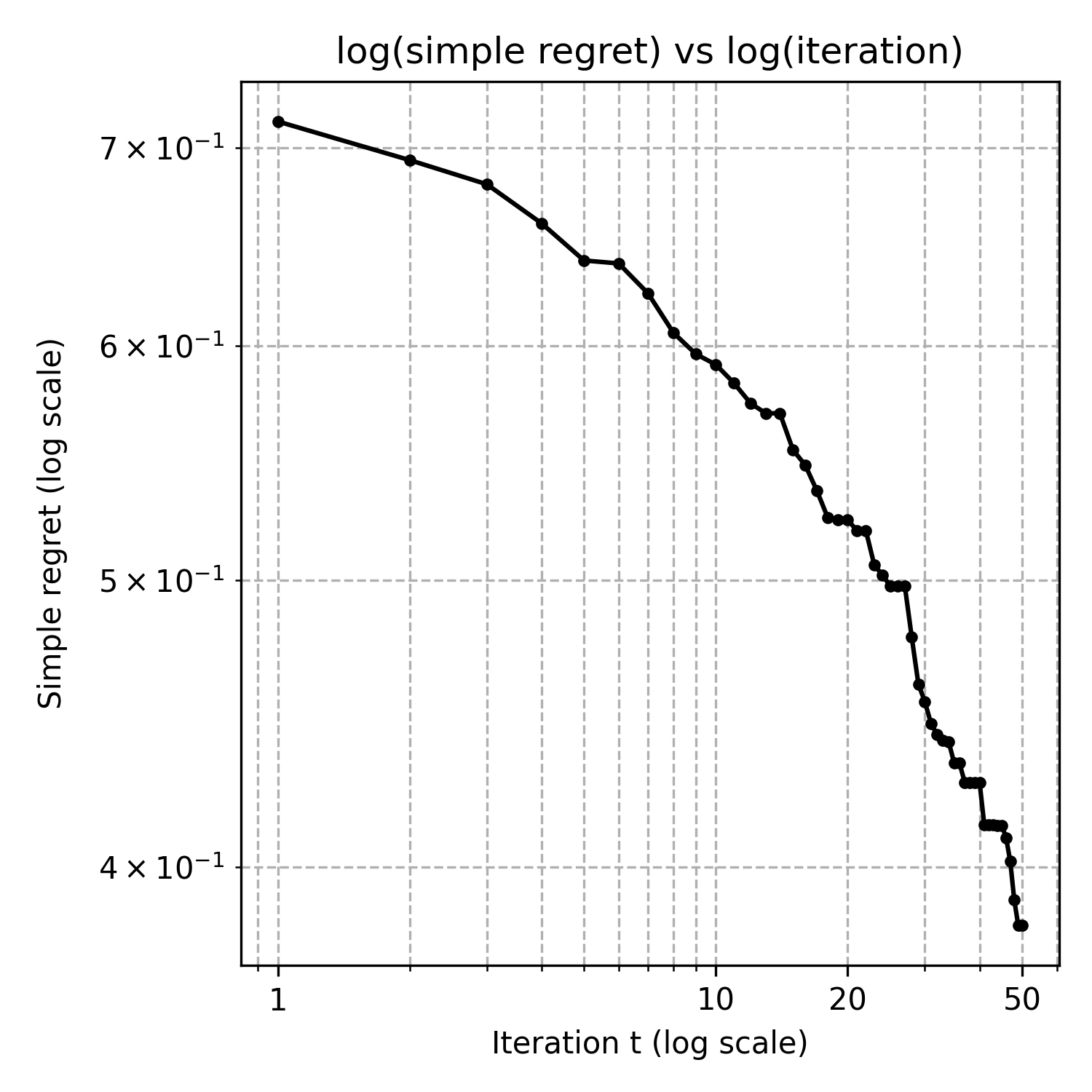}
	\label{fig:s2}
	}
\subfigure[GP, SE, $d=2$]{
	\centering
	\includegraphics[width=0.225\textwidth]{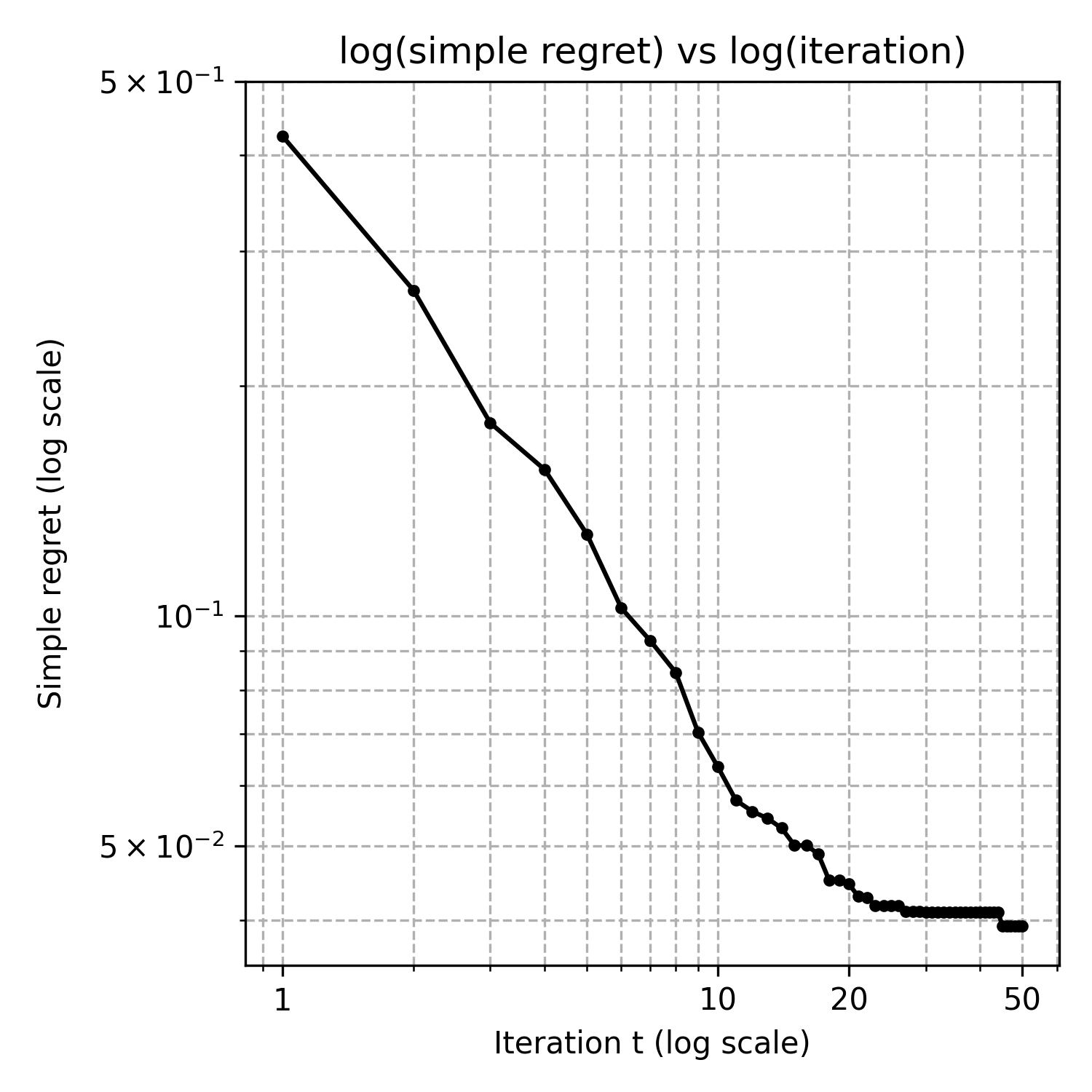}
	\label{fig:3}
}
\subfigure[GP, SE, $d=4$]{
	\centering
	\includegraphics[width=0.225\textwidth]{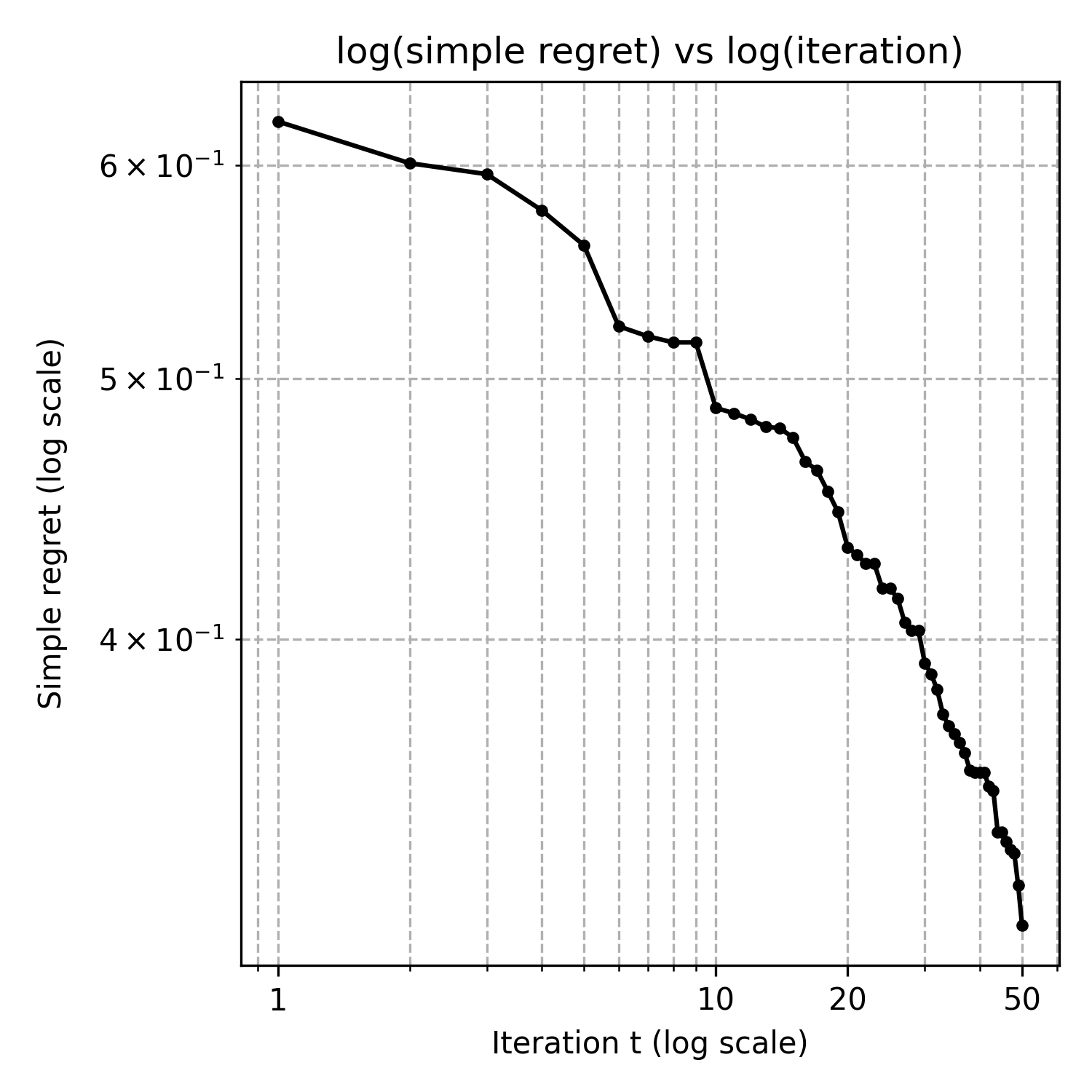}
	\label{fig:4}
}
\caption{The log-log plots for simple regret vs optimization iterations of CEI for the synthetic problems.}
\label{fig:synthetic}
\end{figure}

\subsection{Test problems}\label{se:testproblem}

Next, we evaluate simple regret of CEI on five commonly used test problems in the literature of CBO.  Specifically, Problem 1 tests performance in a small feasible region, which was previously studied in~\citet{gardner2014,ariafar2019admmbo}. 
Problem 2 includes multiple constraints and local minimums, which has been used in~\citet{gramacy2016modeling,hernandez2015predictive}. 
Problem 3 is a four-dimensional problem, previously studied in~\citet{picheny2016bayesian,ariafar2019admmbo}. 
Problem 4 is the six-dimensional Hartmann problem, previously tested in~\citet{letham2019constrained}.
Problem 5 is the Rosenbrock function, where the global minimum lies in a narrow region. The mathematical formulations of the five functions are presented in Appendix~\ref{se:exp_contour}. For two-dimensional problems, we also include the contour plots of the objective and constraint functions in Appendix~\ref{se:exp_contour}. The SE kernel is used for the GP modeling (similar performance is observed for the Matérn kernel) and the hyper-parameters are estimated by a standard maximum likelihood method.

%For each test problem, the number of initial design points is set to $2d$, where $d$ is the input dimension. The number of optimization iterations is chosen based on problem complexity and the convergence behavior of the algorithm. 
For each test problem, we conducted 100 independent trials with  different random initial designs. When CEI fails to identify a feasible sample, we adopt the same heuristic strategy as in~\citet{letham2019constrained}. The numerical results are summarized in Figure \ref{fig:numerical}. The solid line represents the median of the simple regret, and the dotted lines represent the 25th percentile and 75th percentile of the simple regret, respectively. From the figures, we observe that CEI consistently reduces simple regret, aligning with the asymptotic convergence theories established in this paper. Accross all problems, the simple regret converges to $0$ quickly. The 25th percentile and 75th percentile results demonstrate the good statistical properties of CEI.

\begin{figure}[!h]
\subfigure[Problem 1: a small feasible region problem]{
	\centering
	\includegraphics[width=0.47\textwidth]{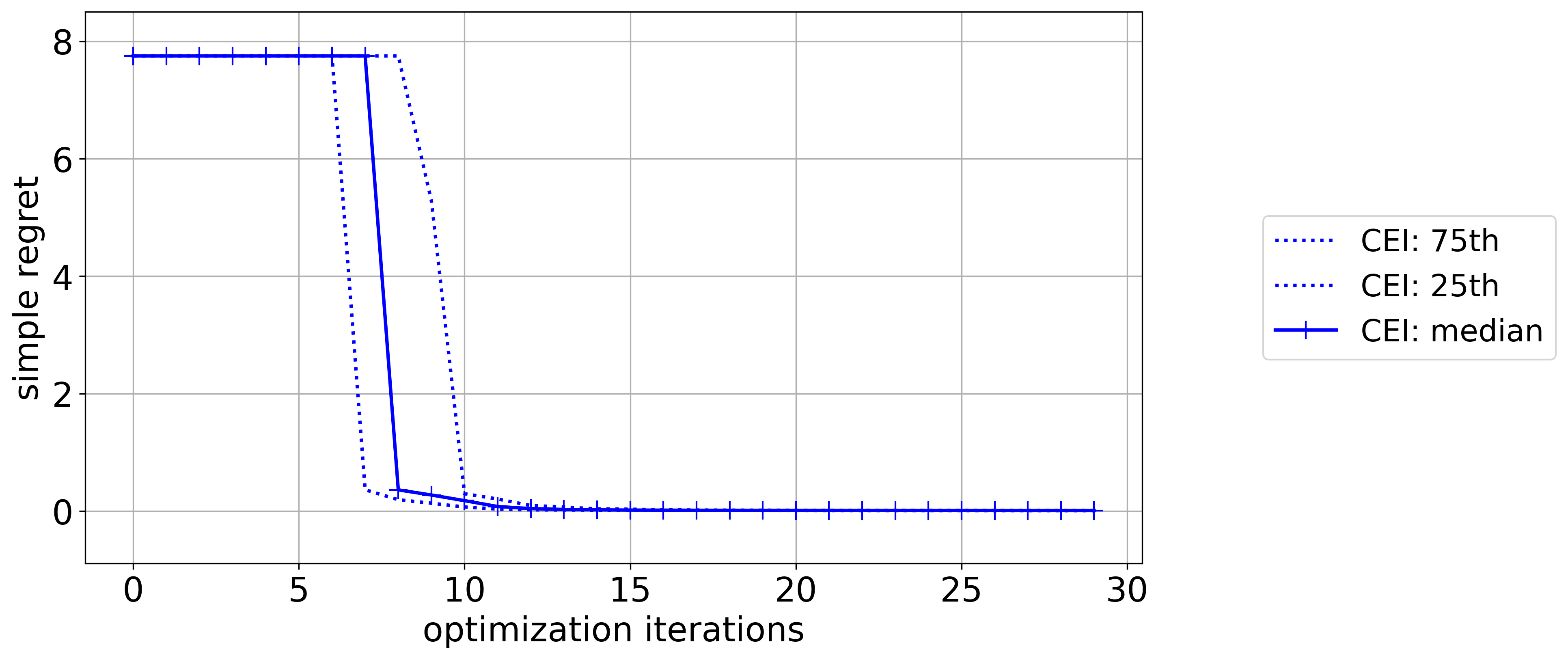}
	\label{fig:ex1}
}
\subfigure[Problem 2: a multiple-constraints problem]{
	\centering
	\includegraphics[width=0.47\textwidth]{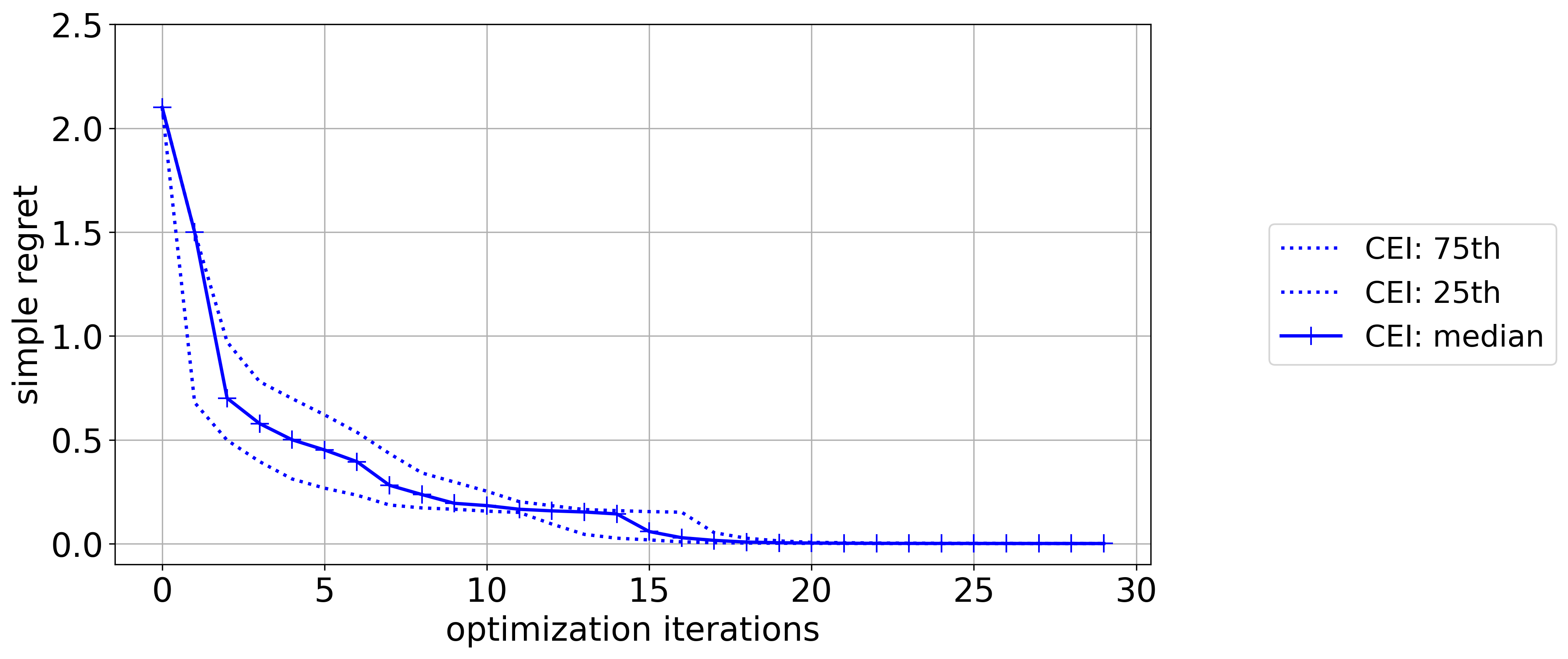}
	\label{fig:ex2}
	}
\subfigure[Problem 3: a four-dimensional problem]{
	\centering
	\includegraphics[width=0.47\textwidth]{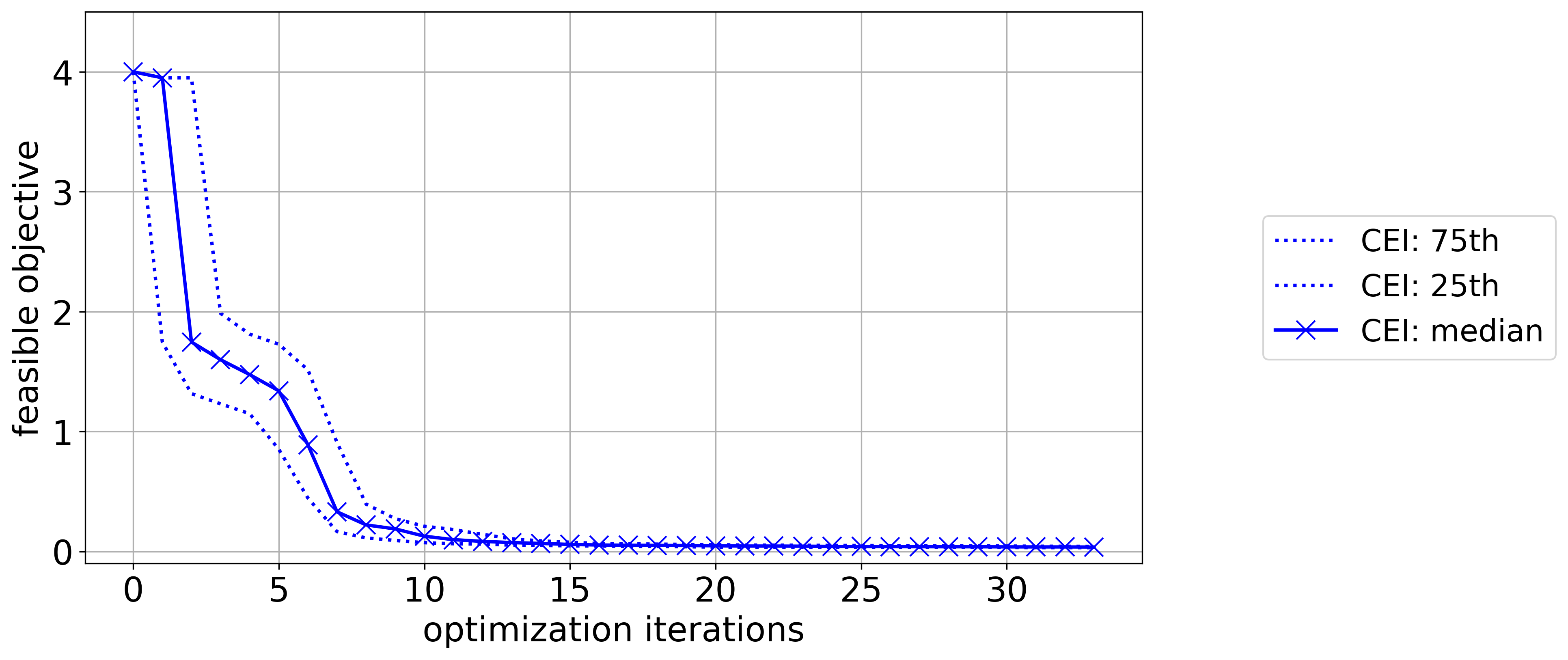}
	\label{fig:ex3}
}
\subfigure[Problem 4: Hartmann6 function]{
	\centering
	\includegraphics[width=0.47\textwidth]{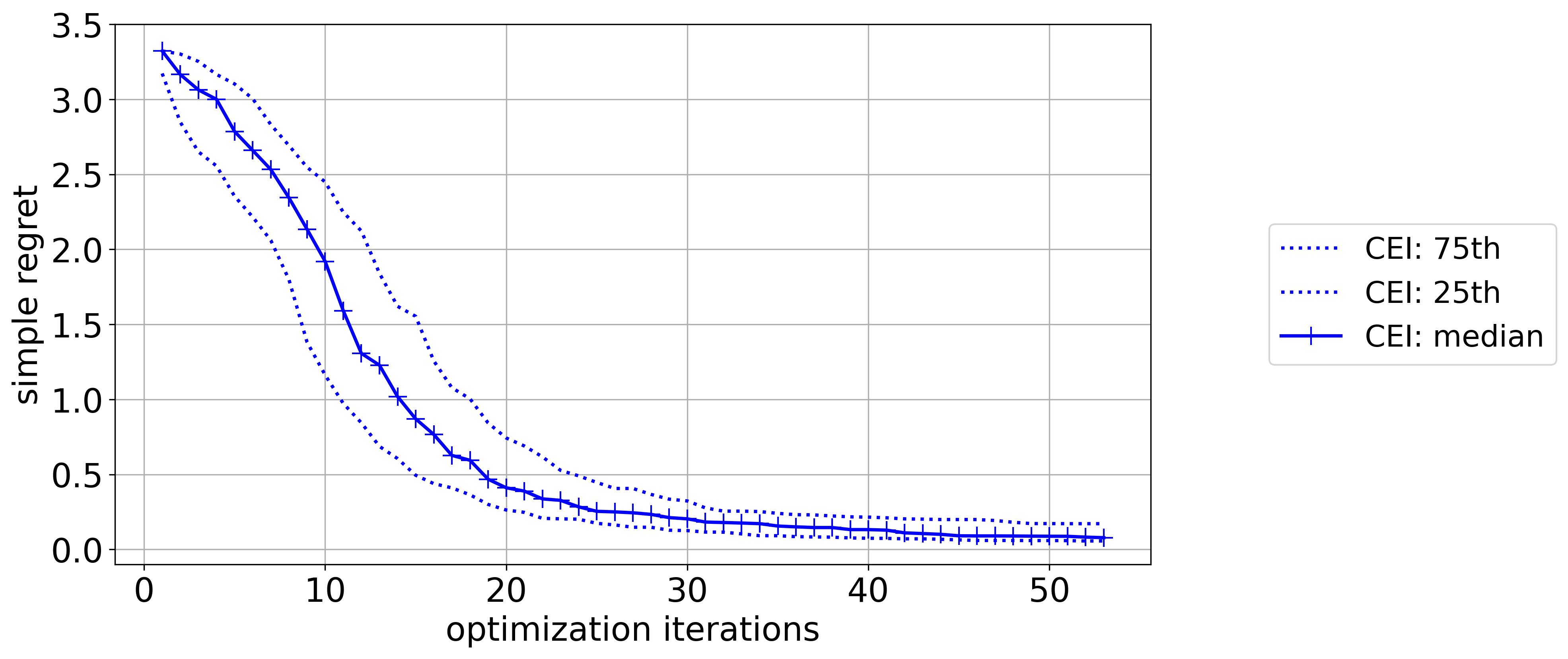}
	\label{fig:ex4}
}
\subfigure[Problem 5: Rosenbrock function (in log scale due to large range of objective values.)]{
	\centering
	\includegraphics[width=0.47\textwidth]{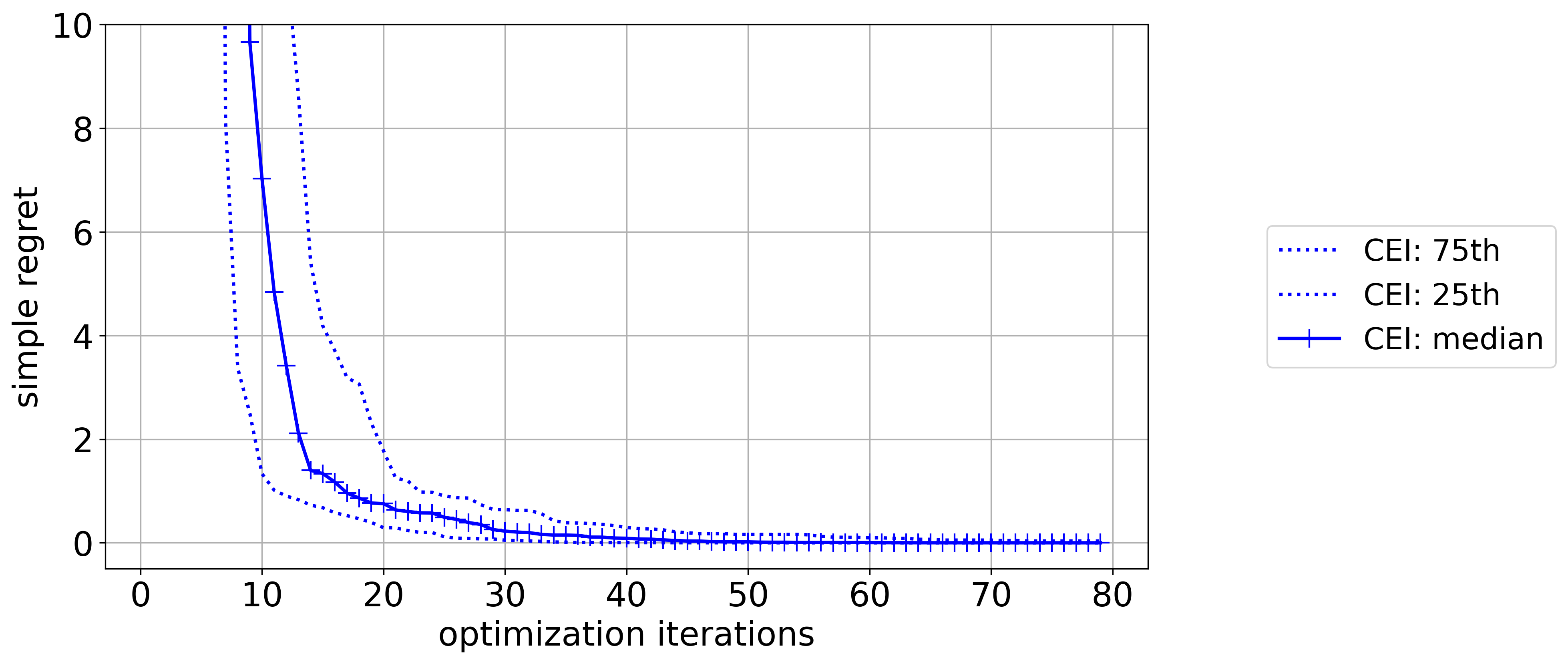}
	\label{fig:ex5}
}
\caption{Simple regret of CEI for five test problems.}
\label{fig:numerical}
\end{figure}

\section{Conclusions}\label{se:conclusion}
In this paper, we studied the simple regret upper bounds of the CEI algorithm, one of the most widely adopted CBO methods. Under both frequentist setting and Bayesian objective assumptions, we establish for the first time the convergence rates for CEI. Our results provide theoretical support and validation for the empirical success of CEI. 

%\begin{ack}
%This work was performed under the auspices of the U.S. Department of Energy by Lawrence Livermore National Laboratory under contract DE-AC52-07NA27344.
%\end{ack}

\medskip
\clearpage
%\newpage
%\setcitestyle{numbers}

\setcitestyle{numbers}
\bibliographystyle{plainnat}
\bibliography{bibliography}

%%%%%%%%%%%%%%%%%%%%%%%%%%%%%%%%%%%%%%%%%%%%%%%%%%%%%%%%%%%%%%%%%%%%%%%%%%%%%%%
%%%%%%%%%%%%%%%%%%%%%%%%%%%%%%%%%%%%%%%%%%%%%%%%%%%%%%%%%%%%%%%%%%%%%%%%%%%%%%%
% APPENDIX
%%%%%%%%%%%%%%%%%%%%%%%%%%%%%%%%%%%%%%%%%%%%%%%%%%%%%%%%%%%%%%%%%%%%%%%%%%%%%%%
%%%%%%%%%%%%%%%%%%%%%%%%%%%%%%%%%%%%%%%%%%%%%%%%%%%%%%%%%%%%%%%%%%%%%%%%%%%%%%%
\newpage
\appendix
\onecolumn

\section{Background information and preliminary results}\label{se:app-bck}
\subsection{Information gain}\label{se:informationgain}

To obtain state-of-the-art simple regret upper bound, we utilize the information theory results that are well-established in previous literature~\citep{srinivas2009gaussian,chowdhury2017kernelized,vakili2021information}.
Using $f$ as the example, let $A\subset C$ denote a set of sampling points. Assume that the observations are noisy at sample points with $y_A = f(\xbm_A)+\epsilon_A$ at $\xbm \in A$, where $\epsilon_A\sim\mathcal{N}(0,\sigma^2)$ denotes the independent and identically distributed Gaussian noises.  
The maximum information gain is defined as follows.
\begin{definition}\label{def:infogain}
	Given $\xbm_A$ and $\ybm_A$, the mutual information between $f$ and $\ybm_A$ is $I(\ybm_A;f_A)=H(\ybm_A)-H(\ybm_A|f_A)$, where $H$ denotes the entropy. The maximum information gain $\gamma_T^f$ after $T$ samples is $\gamma_T^f = \max_{A\subset C,|A|=T} I(\ybm_A;\fbm_A)$. 
\end{definition}

The rate of increase for $\gamma^f_t$ depends on the property of the kernel.
For common kernels such as the SE kernel and the Matérn kernel, $\gamma^f_t$ has been studied in literature and
the state-of-the-art rates of $\gamma^f_t$ are summarized in Lemma~\ref{lem:gammarate}~\citep{vakili2021information,iwazaki2025improved}.
\begin{lemma}\label{lem:gammarate}
  For GP with $t$ samples, the SE kernel has $\gamma_t=\mathcal{O}(\log^{d+1}(t))$, and
  the Matérn kernel with smoothness parameter $\nu>0.5$ has $\gamma_t=\mathcal{O}(t^{\frac{d}{2\nu+d}}(\log^{\frac{2\nu}{2\nu+d}} (t)))$.
\end{lemma}
The maximum information gain $\gamma_t^c$ for the constraint function can be defined similarly.  

While $\gamma_t^f$ is defined in the noisy case, we can readily apply it to the noise-free case and bound $\sigma_{t}^f(\xbm_{t+1})$ using techniques similar to that in~\cite{lyu2019efficient}. 
To do so, we note that given the Gaussian observation noise $\epsilon_t\sim\mathcal{N}(0,\sigma^2)$ in the GP model, the posterior prediction for $f$ becomes
\begin{equation} \label{eqn:GP-post-2}
 \centering
  \begin{aligned}
  &\Tilde{\mu}_t^f(\xbm)\ =\ \kbm_t^f(\xbm) (\Kbm_t^f+ \sigma^2 \Ibm)^{-1} \fbm_{1:t},\\
  &(\tilde{\sigma}_t^f)^2(\xbm)\ =\
k^f(\xbm,\xbm)-(\kbm_t^f)^T(\xbm)(\Kbm_t^f+\sigma^2 \Ibm)^{-1}\kbm_t^f(\xbm)\ ,
\end{aligned}
\end{equation}
Similarly, we can define the posterior predictions for $c$ with noise in the GP model as $\tilde{\mu}_t^f(\xbm)$ and $\tilde{\sigma}_t^f(\xbm)$.
The sum of posterior variance for GP generated by~\eqref{eqn:GP-post} satisfy the next lemma, based on information theory~\cite{srinivas2009gaussian} .
\begin{lemma}\label{lem:variancebound}
 The sum of GP posterior variances given $t$ samples satisfy
 \begin{equation} \label{eqn:var-1}
  \centering
  \begin{aligned}
    \sum_{t=1}^T  \tilde{\sigma}_{t-1}^c(\xbm_t) \leq \sqrt{C_{\gamma} T \gamma^c_T}, \ \sum_{t=1}^T  \tilde{\sigma}_{t-1}^f(\xbm_t) \leq \sqrt{C_{\gamma} T \gamma^f_T},  
  \end{aligned}
\end{equation}
 where $C_{\gamma} = \frac{2}{log(1+\sigma^{-2})}$ and $\gamma_t^f$ and $\gamma_t^c$ are the maximum information gains for $f$ and $c$, respectively. 
\end{lemma}

We have the following lemma for $\tilde{\sigma}_t^f(\xbm)$ and $\sigma_t^f(\xbm)$.
\begin{lemma}\label{lem:sigma-noise}
  The noise-free (GP) posterior standard deviation satisfies
  $\sigma_t^f(\xbm) <\tilde{\sigma}_t^f(\xbm)$ for $\forall \sigma>0$.
\end{lemma}
\begin{proof}
    We first note that all the eigenvalues of $\Kbm_t^f$ is smaller than those of $\Kbm_t^f+\sigma^2 \Ibm$, since $\Kbm_t^f$ is symmetric and positive definite. 
    Thus,
    \begin{equation} \label{eqn:lem:sigma-noise-pf-1}
 \centering
  \begin{aligned}
     (\kbm_t^f)^T(\xbm)(\Kbm_t^f+\sigma^2 \Ibm)^{-1}\kbm_t^f(\xbm) < (\kbm_t^f)^T(\xbm)(\Kbm_t^f)^{-1}\kbm_t^f(\xbm), 
  \end{aligned}
\end{equation}
   for $\forall \sigma>0$ and $\forall \kbm_t^f(\xbm)$.
    Therefore, by their definitions~\eqref{eqn:GP-post} and~\eqref{eqn:GP-post-2}, the proof is complete. 
\end{proof}
The posterior standard deviation under assumptions (1)-(4) in~\cite{bull2011convergence} in the frequentist and noise-free setting is given in Lemma~\ref{lem:EI-sigma}, whose proof is given below.
\begin{proof}
   By Lemma 7 in~\cite{bull2011convergence}, there exists $C'>0$ so that 
   \begin{equation} \label{eqn:EI-sigma-pf-1}
  \centering
  \begin{aligned}
         \sigma_i^f(\xbm_{i+1}) \geq C' k^{-\frac{\min\{\nu,1\}}{d}} \log^{\eta}(k),
  \end{aligned}
  \end{equation}
  at most $k$ times, for $\forall k\in\Nbb$, $k\leq t$, and $i=1,\dots,t-1$.
  Therefore, for the SE kernel, 
   \begin{equation} \label{eqn:EI-sigma-pf-2}
  \centering
  \begin{aligned}
         \sigma_i^f(\xbm_{i+1}) \geq C' k^{-\frac{1}{d}},
  \end{aligned}
  \end{equation}
   at most $k$ times. For  Matérn kernels, we have~\eqref{eqn:EI-sigma-pf-1} at most $k$ times.
\end{proof}

Using maximum information gain, a tighter bound on the posterior standard deviation can be obtained in the next lemma. 
\begin{lemma}\label{lem:EI-sigma-gamma}
 Given $\forall t\in\Nbb$ and $i=1,2,\dots,t-1$, for SE kernel, 
  there exists constant $C'>0$ so that    
  \begin{equation} \label{eqn:EI-sigma-gamma-1}
  \centering
  \begin{aligned}
         \sigma_i^f(\xbm_{i+1}) \geq C' \frac{t^{\frac{1}{2}} \log^{\frac{d+1}{2}}(t)}{k}, 
  \end{aligned}
  \end{equation}
  holds for at most $k$ times,  for $\forall k \leq t$.
  Similarly, for Matérn kernels with $\nu>0.5$,
  \begin{equation} \label{eqn:EI-sigma-gamma-2}
  \centering
  \begin{aligned}
         \sigma_i^f(\xbm_{i+1}) \geq C'  \frac{t^{\frac{\nu+ d}{2\nu+d}}\log^{\frac{\nu}{2\nu+d}} (t)}{k},  
  \end{aligned}
  \end{equation}
  holds at most $k$ times.
\end{lemma}
\begin{proof}
   From Lemma~\ref{lem:sigma-noise} and Lemma~\ref{lem:variancebound}, we can write that  
  \begin{equation} \label{eqn:EI-sigma-gamma-pf-1}
  \centering
  \begin{aligned}
        \sum_{i=1}^t \sigma_{i-1}^f(\xbm_{i}) \leq  \sqrt{t \gamma_t^f},
  \end{aligned}
  \end{equation}
  where we use without losing generality $C_{\gamma}\leq 1$.
  Therefore, for any $k\in\Nbb$ and $k\leq t$,  
  \begin{equation} \label{eqn:EI-sigma-gamma-pf-2}
  \centering
  \begin{aligned}
         \sigma_i^f(\xbm_{i+1}) \geq  \frac{\sqrt{t \gamma_t^f}}{k},
  \end{aligned}
  \end{equation}
  at most $k$ times.
  Therefore, for SE kernel, by Lemma~\ref{lem:gammarate}, there exists $C'>0$ such that  
  \begin{equation} \label{eqn:EI-sigma-gamma-pf-3}
  \centering
  \begin{aligned}
         \sigma_i^f(\xbm_{i+1}) \geq C' \frac{t^{\frac{1}{2}} \log^{\frac{d+1}{2}}(t)}{k}, 
  \end{aligned}
  \end{equation}
   at most $k$ times. For  Matérn kernels with $\nu>0.5$, 
  \begin{equation} \label{eqn:EI-sigma-gamma-pf-4}
  \centering
  \begin{aligned}
         \sigma_i^f(\xbm_{i+1}) \geq  C'  \frac{t^{\frac{\nu+ d}{2\nu+d}}\log^{\frac{\nu}{2\nu+d}} (t)}{k},  
  \end{aligned}
  \end{equation}
  at most $k$ times.
\end{proof}

Next, we state some basic properties of $\phi$, $\Phi$ and $\tau$ as a lemma.
\begin{lemma}\label{lem:phi}
The PDF and CDF of standard normal distribution satisfy $0< \phi(x)\leq \phi(0)<0.4, \Phi(x)\in(0,1)$, 
for any $x\in\Rbb$. 
 Given a random variable $\xi$ sampled from the standard normal distribution: $\xi\sim\mathcal{N}(0,1)$, we have $\Pbb\{\xi>c|c>0\}\leq \frac{1}{2}e^{-c^2/2}$. 
 Similarly, for $c<0$, $\Pbb\{\xi<c|c<0\}\leq \frac{1}{2}e^{-c^2/2}$. 
 The function $\tau(\cdot)$ is monotonically increasing. 
\end{lemma}
The last statement in Lemma~\ref{lem:phi} is a well-known result (\textit{e.g.}, see proof of Lemma 5.1 in~\cite{srinivas2009gaussian}).

The next lemma proves basic properties for $EI_t^f$.
\begin{lemma}\label{lem:EI}
 For $\forall \xbm\in C$, 
    $EI_t^f(\xbm) \geq 0$ and $EI_t^f(\xbm) \geq f^+_{t}- \mu_{t}^f(\xbm)$.
Moreover, 
\begin{equation} \label{eqn:EI-property-2}
 \centering
  \begin{aligned}
   z_{t}^f(\xbm)\leq  \frac{EI_{t}^f(\xbm)}{\sigma_{t}^f(\xbm) } < \begin{cases}  \phi(z_{t}^f(\xbm)), \ &z_{t}^f(\xbm)<0,\\
             z_{t}^f(\xbm) +\phi(z_{t}^f(\xbm)), \ &z_{t}^f(\xbm)\geq 0.
      \end{cases}
  \end{aligned}
\end{equation}
\end{lemma}
\begin{proof}
     From the definition of $I_{t}^f$ and $EI_{t}^f$, $EI_t^f(\xbm) \geq 0$ and $EI_t^f(\xbm) \geq y^+_{t}- \mu_{t}^f(\xbm)$ follow immediately.
     By~\eqref{eqn:EI-1},  
\begin{equation} \label{eqn:EI-property-pf-1}
 \centering
  \begin{aligned}
     \frac{EI_{t}^f(\xbm)}{\sigma_{t}^f(\xbm) } = z_{t}^f(\xbm) \Phi(z_{t}^f(\xbm)) + \phi(z_{t}^f(\xbm)). 
  \end{aligned}
\end{equation}
If $z_{t}^f(\xbm)<0$, or equivalently $f_{t}^+ -\mu_{t}^f(\xbm)<0$,~\eqref{eqn:EI-property-pf-1} leads to
     $\frac{EI_{t}^f(\xbm)}{\sigma_{t}^f(\xbm)} < \phi(z_{t}^f(\xbm))$. 
If $z_{t}^f(\xbm)\geq 0$, we can write
     $\frac{EI_{t}^f(\xbm)}{\sigma_{t}^f(\xbm)} < z_{t}^f(\xbm) +\phi(z_{t}^f(\xbm))$. 
 The left inequality in~\eqref{eqn:EI-property-2} is an immediate result of $EI_{t}^f(\xbm) \geq f^+_{t}- \mu_{t}^f(\xbm)$.
\end{proof}

\section{Proofs for simple regret upper bound under frequentist assumptions}
We state the boundedness result as a Lemma for easy reference. 
\begin{lemma}\label{lem:f-bound}
  Under Assumption~\ref{assp:rkhs}, $|f(\xbm)|\leq B_f$ and $|c(\xbm)|\leq B_c$ for all $\xbm\in C$.
\end{lemma}
\begin{proof}
    By Assumption~\ref{assp:rkhs}, we can write 
    \begin{equation} \label{eqn:f-bound-pf-1}
  \centering
  \begin{aligned}
     |f(\xbm)|\leq \norm{f}_{H_k^f} k_f(\xbm,\xbm)\leq B_f.
  \end{aligned}
\end{equation}
 Similarly, 
  \begin{equation} \label{eqn:f-bound-pf-2}
  \centering
  \begin{aligned}
     |c(\xbm)|\leq \norm{c}_{H_k^c} k_c(\xbm,\xbm)\leq B_c.
  \end{aligned}
\end{equation}
\end{proof}

The following Lemma is a well-established result~\cite{chowdhury2017kernelized}.
\begin{lemma}\label{lem:rkhs-bound-nonoise}
    At any given $\xbm\in C$  and $t\in\Nbb$, the confidence intervals satisfy   
 \begin{equation} \label{eqn:rkhs-bound-f-nonoise-1}
  \centering
  \begin{aligned}
     |f(\xbm) - \mu_{t}^f(\xbm)|\leq B_f \sigma_{t}^f(\xbm), |c(\xbm) - \mu_{t}^c(\xbm)|\leq B_c \sigma_{t}^c(\xbm).
  \end{aligned}
\end{equation}
\end{lemma}
Next, we present a lemma on the relationship between $I_t^f$ and $EI_t^f$, previously seen in~\cite{bull2011convergence}. 
\begin{lemma}\label{lem:I-bound-EI}
   At $\xbm\in C,  t\in\Nbb$,
 \begin{equation} \label{eqn:rkhs-bound-nonoise-1}
  \centering
  \begin{aligned}
     I_t^f(\xbm) - EI_{t}^f(\xbm) \leq B_f \sigma_{t}^f(\xbm).
  \end{aligned}
\end{equation}
\end{lemma}

\begin{lemma}\label{lem:EI-ratio-bound}
  The improvement function $I^f_{t}(\xbm)$ and $EI^f_{t}(\xbm)$ satisfy
  \begin{equation} \label{eqn:EI-ratio-bound-1}
  \centering
  \begin{aligned}
             I_{t}^f(\xbm) \leq \frac{\tau(B_f)}{\tau(-B_f)} EI^f_{t}(\xbm),
  \end{aligned}
  \end{equation}
  for $\forall  \xbm\in C$ and $t\geq 1$.
\end{lemma}
\begin{proof}
   If $f^+_{t}-f(\xbm) \leq 0$, then $I_t^f(\xbm) = 0$. Since $EI_t^f(\xbm)\geq 0$,~\eqref{eqn:EI-ratio-bound-1} is trivial.
   If  $f^+_{t}-f(\xbm) > 0$, by Lemma~\ref{lem:rkhs-bound-nonoise}, 
 \begin{equation} \label{eqn:IEI-bound-ratio-pf-1}
 \centering
  \begin{aligned}
     f^+_{t}-\mu_{t}^f(\xbm) &=     f^+_{t}-f(\xbm)+f(\xbm)-\mu_{t}^f(\xbm) >f(\xbm)-\mu_{t}^f(\xbm)>- B_f\sigma_{t}^f(\xbm).
  \end{aligned}
\end{equation}
  From the monotonicity of $\tau$, we have 
 \begin{equation} \label{eqn:IEI-bound-ratio-pf-1.5}
 \centering
  \begin{aligned}
     \tau\left(\frac{f^+_{t}-\mu_{t}^f(\xbm)}{\sigma_{t}^f(\xbm)}\right) >\tau(-B_f),
  \end{aligned}
\end{equation}
 and therefore,  
 \begin{equation} \label{eqn:IEI-bound-ratio-pf-2}
 \centering
  \begin{aligned}
     EI_t^f(\xbm) = \sigma_t^f(\xbm)\tau\left(\frac{f^+_{t}-\mu_{t}^f(\xbm)}{\sigma_{t}^f(\xbm)}\right) >\tau(-B_f)\sigma_{t}^f(\xbm).
  \end{aligned}
\end{equation}
  From Lemma~\ref{lem:I-bound-EI}, 
  \begin{equation} \label{eqn:IEI-bound-ratio-pf-4}
  \centering
  \begin{aligned}
      I_t^f(\xbm)- EI_{t}^f(\xbm) \leq B_f\sigma_{t}^f(\xbm). 
  \end{aligned}
\end{equation}
  Applying~\eqref{eqn:IEI-bound-ratio-pf-4} to~\eqref{eqn:IEI-bound-ratio-pf-2} leads to 
 \begin{equation} \label{eqn:IEI-bound-ratio-pf-5}
 \centering
  \begin{aligned}
     EI_t^f(\xbm)  > \frac{\tau(-B_f)}{B_f+\tau(-B_f)} I_t^f(\xbm) = \frac{\tau(-B_f)}{\tau(B_f)} I_t^f(\xbm).
  \end{aligned}
\end{equation}
\end{proof}

Proof of Theorem~\ref{theorem:CEI-convg-nonoise} is given next.
\begin{proof}
  From Lemma~\ref{lem:f-bound}, 
  \begin{equation} \label{eqn:CEI-convg-nonoise-pf-1}
  \centering
  \begin{aligned}
        \sum_{i=0}^{t-1}  f^+_{i}-f^+_{i+1} = f^+_0 - f_t^+ \leq 2 B_f. 
  \end{aligned}
  \end{equation}
  Since $f^+_{i}-f^+_{i+1}\geq 0$, $f^+_{i}-f^+_{i+1} \geq \frac{2 B_f }{k}$ at most $k$ times for any $k\in \Nbb$. 
  Further, $f(\xbm_{t})\geq f_{t}^+$ for $\forall t\in\Nbb$.
  Choose $k=[t/2]$, where $[x]$ is the largest integer smaller than $x$ so that $2k\leq t\leq 2(k+1)$.
  Then, there exists $k \leq t_k \leq 2k$ so that $f_{t_k}^+-f_{t_k+1}^+<\frac{2B_f}{k}$ and $f_{t_{k}+1}^+-f(\xbm_{t_k+1}) \leq 0$.

  From Lemma~\ref{lem:EI-ratio-bound}, 
  \begin{equation} \label{eqn:CEI-convg-nonoise-pf-2}
  \centering
  \begin{aligned}
          r_t =& f^+_{t}-f(\xbm^*)\leq f^+_{t_k}-f(\xbm^*)\leq I_{t_k}^f(\xbm^*)\\
               \leq& \frac{\tau(B_f)}{\tau(-B_f)}  EI_{t_k}^f(\xbm^*) 
               = c_{\tau B} \frac{P_{t_k}(\xbm^*)}{P_{t_k}(\xbm^*)} EI_{t_k}^f(\xbm^*) \\
               \leq& c_{\tau B} \frac{P_{t_k}(\xbm_{t_k+1})}{P_{t_k}(\xbm^*)} EI_{t_k}^f(\xbm_{t_k+1}), \\
   \end{aligned}
  \end{equation}
  where $c_{\tau B}=\frac{\tau(B_f)}{\tau(-B_f)} $. 
  Using $P_t(\xbm)\leq 1$,~\eqref{eqn:CEI-convg-nonoise-pf-2} implies
  \begin{equation} \label{eqn:CEI-convg-nonoise-pf-3}
  \centering
  \begin{aligned}
          r_t 
                \leq&  \frac{c_{\tau B}}{P_{t_k}(\xbm^*)} \left[(f^+_{t_k}-\mu_{t_k}^f(\xbm_{t_k+1}))\Phi(z_{t_k}^f(\xbm_{t_k+1}))+\sigma_{t_k}^f(\xbm_{t_k+1})\phi(z_{t_k}^f(\xbm_{t_k+1}))\right]  \\
            \leq& \frac{c_{\tau B}}{P_{t_k}(\xbm^*)}\left[(f^+_{t_k}-\mu_{t_k}^f(\xbm_{t_k+1}))\Phi(z_{t_k}^f(\xbm_{t_k+1}))+0.4\sigma_{t_k}^f(\xbm_{t_k+1})\right],
   \end{aligned}
  \end{equation}
   where the last inequality uses $\phi(\cdot)<0.4$.
  From Lemma~\ref{lem:rkhs-bound-nonoise}, 
 \begin{equation} \label{eqn:CEI-convg-nonoise-pf-4}
  \centering
  \begin{aligned}
          f^+_{t_k}-\mu_{t_k}^f(\xbm_{t_k+1}) =& f^+_{t_k}-f^+_{t_k+1}+f^+_{t_k+1}-f(\xbm_{t_k+1})+f(\xbm_{t_k+1})-\mu_{t_k}^f(\xbm_{t_k+1}) \\
          \leq& f^+_{t_k}-f^+_{t_k+1}+B_f\sigma_{t_k}^f(\xbm_{t_k+1})\leq \frac{2B_f}{k} + B_f\sigma_{t_k}^f(\xbm_{t_k+1}).
    \end{aligned}
  \end{equation}
  Using~\eqref{eqn:CEI-convg-nonoise-pf-4} in~\eqref{eqn:CEI-convg-nonoise-pf-3}, we have 
  \begin{equation} \label{eqn:CEI-convg-nonoise-pf-5}
  \centering
  \begin{aligned}
          r_t 
            \leq& \frac{c_{\tau B}}{P_{t_k}(\xbm^*)} \left[ \frac{2B_f}{k} + (B_f+0.4)\sigma_{t_k}^f(\xbm_{t_k+1})\right].
   \end{aligned}
  \end{equation}

   Next, we consider the function $P_{t_k}$ at $\xbm^*$.
   Using the fact that $c(\xbm^*) \leq 0$, we have by Lemma~\ref{lem:rkhs-bound-nonoise},
   \begin{equation} \label{eqn:CEI-convg-nonoise-pf-6}
  \centering
  \begin{aligned}
     \mu_{t_k}^c(\xbm^*)\leq B_c \sigma_{t_k}^c(\xbm^*)+ c(\xbm^*) \leq  B_c \sigma_{t_k}^c(\xbm^*). 
    \end{aligned}
  \end{equation}
   Thus, 
  \begin{equation} \label{eqn:CEI-convg-nonoise-pf-7}
  \centering
  \begin{aligned}
     \frac{-\mu_{t_k}^c(\xbm^*)}{\sigma_{t_k}^c(\xbm^*)} \geq -B_c .
    \end{aligned}
  \end{equation}
  From the monotonicity of $\Phi$, we have 
   \begin{equation} \label{eqn:CEI-convg-nonoise-pf-8}
  \centering
  \begin{aligned}
    P_{t_k}(\xbm^*) =\Phi\left(\frac{-\mu_{t_k}^c(\xbm^*)}{\sigma_{t_k}^c(\xbm^*)}\right) \geq \Phi\left(-B_c\right). 
    \end{aligned}
  \end{equation}
  Applying~\eqref{eqn:CEI-convg-nonoise-pf-8} to~\eqref{eqn:CEI-convg-nonoise-pf-3}, we have
    \begin{equation} \label{eqn:CEI-convg-nonoise-pf-9}
  \centering
  \begin{aligned}
          r_t  \leq&  \frac{c_{\tau B}}{\Phi(-B_c) }\left[\frac{2B_f}{k}+(B_f+0.4)\sigma_{t_k}^f(\xbm_{t_k+1})\right].\\
  \end{aligned}
  \end{equation}
  As $t\to\infty$, $t_k\to\infty$ and $k\to\infty$. Further, if $\sigma_t^f(\xbm_{t+1})\to 0$, $\sigma_{t_k}^f(\xbm_{t_k+1})\to 0$ and $r_t\to 0$.  
  \end{proof}
Proof of Corollary~\ref{prop:EI-convg-rate-nonoise} is presented next.
\begin{proof}
  We consider the convergence rate for $r_t$ under additional assumptions for the kernel.  
  From Lemma~\ref{lem:EI-sigma}, for both SE and  Matérn kernels,
  $\sigma_{i}^f(\xbm_{i+1})  \geq C' k^{-\frac{\min\{\nu,1\}}{d}} \log^{\eta}(k) $ at most $k$ times for any $k\in \Nbb$ and $i=1,\dots,t$.
  
  Choose $k=[t/3]$ so that $3k\leq t\leq 3(k+1)$.
  Following the proof of Theorem~\ref{theorem:CEI-convg-nonoise}, there exists $k\leq t_k\leq 3k$ where $f^+_{t_k}-f^+_{t_k+1} < \frac{2B_f}{k}$, $f(\xbm_{t_k+1})\geq f_{t_k+1}^+$, and $\sigma_{t_k}^f(\xbm_{t_k+1})< C' k^{-\frac{\min\{\nu,1\}}{d}} \log^{\eta}(k)$.
  Similar to~\eqref{eqn:CEI-convg-nonoise-pf-9}, we can obtain
  \begin{equation} \label{eqn:CEI-convg-nonoise-pf-10}
  \centering
  \begin{aligned}
          r_t  \leq&  \frac{c_{B_f}}{\Phi(-B_c) }\left[\frac{2B_f}{k}+(B_f+0.4)  C' k^{-\frac{\min\{\nu,1\}}{d}} \log^{\eta}(k) \right].\\
  \end{aligned}
  \end{equation}
  The convergence rates follow.
\end{proof}

\subsection{Proofs for improved simple regret upper bound under frequentist assumptions}
Proof of Theorem~\ref{prop:EI-convg-rate-nonoise-gamma} is given below.
\begin{proof}
  From the proof of Lemma~\ref{lem:EI-sigma-gamma}, 
  we know $\sigma_i^f(\xbm_{i+1}) \geq \frac{ \sqrt{t \gamma_t^f}}{k}$ at most $k$ times for any $k\leq t$ and $i=1,\dots,t$.  
 
  Choose $k=[t/3]$ so that $3k\leq t\leq 3(k+1)$.
  Following the proof of Theorem~\ref{theorem:CEI-convg-nonoise}, there exists $k\leq t_k\leq 3k$ where $f^+_{t_k}-f^+_{t_k+1} < \frac{2B_f}{k}$, $f(\xbm_{t_k+1})\geq f_{t_k+1}^+$, and $\sigma_{t_k}^f(\xbm_{t_k+1})<  3\frac{\sqrt{t \gamma_t^f}}{t-3}$.
  Similar to~\eqref{eqn:CEI-convg-nonoise-pf-10}, we can obtain
  \begin{equation} \label{eqn:CEI-convg-nonoise-pf-11}
  \centering
  \begin{aligned}
          r_t  \leq&  \frac{c_{B_f}}{\Phi(-B_c) }\left[3\frac{2B_f}{t-3}+3(B_f+0.4)  \frac{\sqrt{t \gamma_t^f}}{t-3} \right].\\
  \end{aligned}
  \end{equation}
  The convergence rates of the simple regret upper bound follow from Lemma~\ref{lem:EI-sigma-gamma}.
\end{proof}

We provide the sample complexity of Theorem~\ref{prop:EI-convg-rate-nonoise-gamma} below.
\begin{corollary}\label{cor:EI-convg-complexity-1}
	Under Assumption~\ref{assp:rkhs}, the CEI algorithm achieves a $\epsilon$ sample complexity of
	\begin{equation} \label{eqn:CEI-convg-complexity-1}
		\centering
		\begin{aligned}
			\mathcal{O}\!\left( 
\frac{1}{\epsilon^{2}} [\log(1/\epsilon)]^{d+1} \right) \ \text{and }\ \mathcal{O}\!\left(
\epsilon^{-\frac{2\nu+d}{\nu}} \,\log(1/\epsilon)
\right),
		\end{aligned}
	\end{equation}
    for SE and Matérn kernels, respectively. 
\end{corollary}
\begin{proof}
    Using Theorem~\ref{prop:EI-convg-rate-nonoise-gamma},
to achieve simple regret of most $\epsilon$ for SE kernel, set
\[
t^{-\frac{1}{2}} \log^{\frac{d+1}{2}}(t) = \epsilon .
\]
Solving asymptotically for $t$ gives the sample complexity
\[
t(\epsilon) = \mathcal{O}\!\left( 
\frac{1}{\epsilon^{2}} [\log(1/\epsilon)]^{d+1} 
\right).
\]
Similarly, for Matérn  kernels, we have 
\begin{equation} \label{eqn:CEI-convg-complexity-pf-2}
		\centering
		\begin{aligned}
		\epsilon = t^{\frac{-\nu}{2\nu+d}} \log^{\frac{\nu}{2\nu+d}}(t),
		\end{aligned}
	\end{equation}
    which completes the proof.
\end{proof}

\section{Proofs for simple  regret upper bound under Bayesian objective assumption}

We state the boundedness of $f$ and $c$ as a Lemma for easy reference. 
\begin{lemma}\label{lem:f-bound-bayesian}
  Under Assumption~\ref{assp:GP}, there exists $M_f>0$  such that $|f(\xbm)|\leq M_f$ with probability $\geq 1-\delta/3$.
  The constraint function is bounded by its RKHS norm bound $|c(\xbm)|\leq B_c$. 
\end{lemma}

We recall a well-known result on confidence interval of $|f(\xbm)-\mu_t^f(\xbm)|$ under Assumption~\ref{assp:GP}~\cite{srinivas2009gaussian}.
\begin{lemma}\label{lem:fmu}
Given $\delta\in(0,1)$, let  $\beta = 2\log(1/\delta)$. For any given $\xbm\in C$ and $t\in\Nbb$,   
\begin{equation} \label{eqn:fmu-1}
 \centering
  \begin{aligned}
     |f(\xbm) - \mu_{t}^f(\xbm)  | \leq \sqrt{\beta} \sigma_{t}^f(\xbm),
  \end{aligned}
\end{equation}
holds with probability $\geq 1-\delta$. 
\end{lemma}
\begin{proof}
 We prove the inequalities for $f$. Under Assumption~\ref{assp:GP}, $f(\xbm)\sim \mathcal{N}(\mu_{t}(\xbm),\sigma_{t}^2(\xbm))$. By Lemma~\ref{lem:phi}, 
 \begin{equation} \label{eqn:fmu-a-pf-1}
 \centering
  \begin{aligned}
     \Pbb\left\{ f(\xbm)-\mu_t^f(\xbm) > \sqrt{\beta} \sigma_t^f(\xbm)\right\} = 1-\Phi\left(\sqrt{\beta} \right) \leq \frac{1}{2} e^{-\frac{\beta}{2}}.
  \end{aligned}
 \end{equation}
  Similarly, 
 \begin{equation} \label{eqn:fmu-a-pf-2}
 \centering
  \begin{aligned}
     \Pbb\left\{ f(\xbm)-\mu_t^f(\xbm) < -\sqrt{\beta} \sigma_t^f(\xbm)\right\} \leq \frac{1}{2} e^{-\frac{\beta}{2}}.
  \end{aligned}
 \end{equation}
  Thus, 
   \begin{equation} \label{eqn:fmu-a-pf-3}
 \centering
  \begin{aligned}
     \Pbb\left\{ |f(\xbm)-\mu_t^f(\xbm)| < \sqrt{\beta} \sigma_t^f(\xbm)\right\} \geq 1- e^{-\frac{\beta}{2}}.
  \end{aligned}
 \end{equation}
  Let 
     $e^{-\frac{\beta}{2}} = \delta$
  and~\eqref{eqn:fmu-1} is proven. 
\end{proof}
\begin{lemma}\label{lem:fmu-t}
Given $\delta\in(0,1)$, let  $\beta = 2\log(\pi_t/\delta)$, where $\pi_t=\frac{\pi^2t^2}{6}$. Then, for all $t\in\Nbb$, 
\begin{equation} \label{eqn:fmu-t-1}
 \centering
  \begin{aligned}
     |f(\xbm) - \mu_{t}^f(\xbm)  | \leq \sqrt{\beta_t} \sigma_{t}^f(\xbm),
  \end{aligned}
\end{equation}
holds with probability $\geq 1-\delta$. 
\end{lemma}

The next lemma address $I_t^f$ under the Bayesian assumption.
\begin{lemma}\label{lem:Icdf}
 Under Assumption~\ref{assp:GP}, the probability distribution of $I_t^f$ satisfies
\begin{equation} \label{eqn:Icdf-1}
 \centering
  \begin{aligned}
     \Pbb\{I_t^f(\xbm) \leq a\} = \begin{cases}
                    0,   \ &a<0,\\
                  \Phi\left(\frac{a}{\sigma_{t}^f(\xbm)}-z_{t}^f(\xbm)\right), \  &a\geq 0.
                  \end{cases}
  \end{aligned}
\end{equation}
\end{lemma}
\begin{proof}
   Under Assumption~\ref{assp:GP}, at a given $t$, $f(\xbm)\sim\mathcal{N}(\mu_{t}^f(\xbm),\sigma_{t}^f(\xbm))$. Since $I_t^f(\xbm)\geq 0$ for all $\xbm$,~\eqref{eqn:Icdf-1} follows immediately if $a<0$.
   For $a\geq 0$,
\begin{equation*} \label{eqn:Icdf-pf-1} 
 \centering
  \begin{aligned}
   \Pbb\{I_t^f(\xbm) \leq a\} = \Pbb\{f^+_{t}-f(\xbm) \leq a\} = 1- \Pbb\{f(\xbm) \leq f^+_{t}-a\}.
  \end{aligned}
\end{equation*}
  Using basic properties of the standard normal CDF,
\begin{equation*} \label{eqn:Icdf-pf-2} 
 \centering
  \begin{aligned}
   1-\Pbb\{f(\xbm) \leq f^+_{t}-a\}=1-\Phi\left(\frac{f_t^+-a-\mu_t^f(\xbm)}{\sigma_t^f(\xbm)} \right)=\Phi\left(\frac{a-f_t^++\mu_t^f(\xbm)}{\sigma_t^f(\xbm)} \right).
  \end{aligned}
\end{equation*}
\end{proof}

Next, we present the relationship between $I_t^f(\xbm)$ and $EI_t^f(\xbm)$.
\begin{lemma}\label{lem:IEIbound}
Given $\delta\in(0,1)$, let $\beta = \max\{1.44,2\log(c_{\alpha}/\delta)\}$, where constant $c_{\alpha} = \frac{1+2\pi}{2\pi}$. Under Assumption~\ref{assp:GP}, at given $ \xbm\in C$ and $t\in \Nbb$,   
\begin{equation} \label{eqn:IEI-bound-1}
 \centering
  \begin{aligned}
     \Pbb\left\{|I_t^f(\xbm) - EI_{t}^f(\xbm)  | \leq \sqrt{\beta} \sigma_{t}^f(\xbm)\right\}\geq 1-\delta. 
  \end{aligned}
\end{equation}
\end{lemma}
\begin{proof}
Given a scalar $w>1$, we consider the probabilities
\begin{equation} \label{eqn:IEI-bound-pf-3}
  \centering
  \begin{aligned}
      \Pbb\left\{ I_t^f(\xbm) > \sigma_{t}^f(\xbm) w + EI_{t}^f(\xbm)\right\} \quad \text{and} \quad  \Pbb\left\{I_t(\xbm)< -\sigma_{t}^f(\xbm) w +EI_{t}^f(\xbm)\right\}.
  \end{aligned}
\end{equation}
Consider the first probability in~\eqref{eqn:IEI-bound-pf-3}.
From Lemma~\ref{lem:EI}, $EI_{t}(\xbm)\geq 0$ for $\forall \xbm$ and $t$. Therefore, $\sigma_{t}(\xbm) w + EI_{t}(\xbm) > 0$. From Lemma~\ref{lem:EI}, Lemma~\ref{lem:Icdf},  and the monotonicity of $\Phi$, we have
 \begin{equation} \label{eqn:IEI-bound-pf-4}
  \centering
  \begin{aligned}
       \Pbb\left\{ I_t^f(\xbm) > \sigma_{t}^f (\xbm) w + EI_{t}^f(\xbm)\right\} =&  1-\Phi\left(\frac{\sigma_{t}^f(\xbm) w+EI_{t}^f(\xbm)-f^+_{t}+ \mu_{t}^f(\xbm)}{\sigma_{t}^f(\xbm)} \right) \\
               \leq& 1-\Phi(w) \leq \frac{1}{2} e^{-\frac{w^2}{2}},
  \end{aligned}
\end{equation}
where the last inequality is from Lemma~\ref{lem:phi}.

For the second probability in~\eqref{eqn:IEI-bound-pf-3}, we further distinguish between two cases.
First, consider $-\sigma_{t}^f(\xbm) w +EI_{t}^f(\xbm)<0$. From Lemma~\ref{lem:Icdf},
 \begin{equation} \label{eqn:IEI-bound-pf-5}
  \centering
  \begin{aligned}
        \Pbb\left\{I_t^f(\xbm)< -\sigma_{t}^f (\xbm) w +EI_t(\xbm)\right\} = 0.
  \end{aligned}
\end{equation}
Second,  consider the premise $ -\sigma_{t}^f(\xbm) w+ EI_{t}^f(\xbm)\geq 0$. By Lemma~\ref{lem:Icdf}, we have
  \begin{equation} \label{eqn:IEI-bound-pf-6}
  \centering
  \begin{aligned}
       \Pbb\left\{I_t^f(\xbm)< -\sigma_{t}^f(\xbm) w +EI_{t}^f(\xbm)\right\} = \Phi\left( - w+\frac{EI_{t}^f(\xbm)-f^+_{t}+\mu_{t}^f(\xbm)}{\sigma_{t}^f(\xbm)} \right).
  \end{aligned}
\end{equation}
To proceed, we show that $f^+_{t}-\mu_{t}^f(\xbm) \geq 0$. Suppose on the contrary, $f^+_{t}-\mu_{t}^f(\xbm)<0$ and thus $z_{t}^f(\xbm)<0$. From Lemma~\ref{lem:EI}, 
  \begin{equation} \label{eqn:IEI-bound-pf-7}
  \centering
  \begin{aligned}
   \frac{EI_{t}^f(\xbm)}{ \sigma_{t}^f(\xbm)}  < \phi(z_{t}^f(\xbm)) \leq \phi(0) < 1 \leq w,
  \end{aligned}
  \end{equation}
  which contradicts the premise of this case.
Thus, we have $f^+_{t}-\mu_{t}^f(\xbm)\geq 0$ (and $z_{t}^f(\xbm)\geq 0$).
From the definition~\eqref{eqn:EI-1}, since $\Phi\in (0,1)$,
  \begin{equation} \label{eqn:IEI-bound-pf-8}
  \centering
  \begin{aligned}
      \frac{EI_{t}^f(\xbm)-f^+_{t}+\mu_{t}^f(\xbm)}{\sigma_{t}^f(\xbm)}=& \left[z_{t}^f(\xbm)\left(\Phi(z_{t}^f(\xbm))-1\right)+\phi(z_{t}^f(\xbm))\right]           <\phi(z^f_{t}(\xbm)).
  \end{aligned}
  \end{equation}
In addition, by the premise of this case and Lemma~\ref{lem:EI},
  \begin{equation} \label{eqn:IEI-bound-pf-9}
  \centering
  \begin{aligned}
    w\leq  \frac{EI_{t}^f(\xbm)}{\sigma_{t}^f(\xbm)}
       \leq z_{t}^f(\xbm)+\phi(z_{t}^f(\xbm)).\\
  \end{aligned}
  \end{equation}
 Given that $w>1$ and $\phi(0)\geq \phi(z_{t}^f(\xbm))$, we have
  \begin{equation} \label{eqn:IEI-bound-pf-9.5}
  \centering
  \begin{aligned}
   z_{t}^f(\xbm)+\phi(0)>z_{t}^f(\xbm)+\phi(z_{t}^f(\xbm))>w, \ z_{t}^f(\xbm)>w-\phi(0)>0.
  \end{aligned}
  \end{equation}
  As $z_{t}(\xbm)\geq 0$ increases, $\phi(z_{t}(\xbm))>0$ decreases. Thus, we have
\begin{equation} \label{eqn:IEI-bound-pf-10}
 \centering
  \begin{aligned}
    \frac{z_{t}^f(\xbm)}{\phi(z_{t}^f(\xbm))} > \frac{w-\phi(0)}{\phi(w-\phi(0))}, 
    \ \phi(z_{t}^f(\xbm))< \frac{\phi(w-\phi(0))}{w-\phi(0)} z_{t}^f(\xbm) .
  \end{aligned}
\end{equation}
  Denote $c_1(w) = \frac{w-\phi(0)}{w-\phi(0)+\phi(w-\phi(0))}$.
  Applying~\eqref{eqn:IEI-bound-pf-10} to~\eqref{eqn:IEI-bound-pf-9}, we obtain
  \begin{equation} \label{eqn:IEI-bound-pf-11}
  \centering
  \begin{aligned}
    c_1(w) w < z_{t}^f(\xbm), \ \ \phi(z_{t}^f(\xbm)) < \phi( c_1(w) w) .\\
  \end{aligned}
  \end{equation}
  Applying~\eqref{eqn:IEI-bound-pf-11} and~\eqref{eqn:IEI-bound-pf-8} to~\eqref{eqn:IEI-bound-pf-6}, we obtain
  \begin{equation} \label{eqn:IEI-bound-pf-12}
  \centering
  \begin{aligned}
       \Pbb\left\{I_t^f(\xbm)< - w\sigma_t^f(\xbm) + EI_{t}^f(\xbm)\right\} <& \Phi\left( -w+ \phi(z_t^f(\xbm)) \right) \\
   <& \Phi\left( -w+ \phi(c_1(w) w) \right).
  \end{aligned}
\end{equation}
 Notice that $\phi(c_1(w) w)<\phi(c_1(w)) < \phi(c_1(w)) w$ due to $w>1$.
By the definition of $\Phi$ and mean value theorem,
  \begin{equation} \label{eqn:IEI-bound-pf-13}
  \centering
  \begin{aligned}
        &\Phi\bigl(-w+  \phi(c_1(w) w) \bigr) = \Phi(-w)+ \int_{- w}^{-w+\phi(c_1(w) w)} \frac{1}{\sqrt{2\pi}} e^{-\frac{1}{2}x^2} dx
         \leq \Phi(-w)+\\&\frac{1}{\sqrt{2\pi}} e^{-\frac{1}{2} (w- \phi(c_1(w) w))^2} \phi(c_1(w) w)
         \leq \Phi(- w)+\frac{1}{2\pi} e^{-\frac{1}{2} ((1-\phi(c_1(w))) w)^2}  e^{-\frac{1}{2} (c_1(w) w)^2} \\
          &\leq \Phi(-w)+\frac{1}{2\pi} e^{-\frac{1}{2} c_2(w) w^2} 
          \leq \frac{1}{2} e^{-\frac{1}{2} w^2} + \frac{1}{2\pi} e^{-\frac{1}{2}c_2(w) w^2}, 
  \end{aligned}
  \end{equation}
  where $c_2(w)=[1-\phi(c_1(w))]^2+[c_1(w)]^2$.
The last inequality in~\eqref{eqn:IEI-bound-pf-13} again uses Lemma~\ref{lem:phi}.
  Notice that $c_2(w)$ increases with $w$ and for $w\geq 1.2$, $c_2(w)>1$.
  Thus, $e^{-\frac{1}{2} w^2} > e^{-\frac{1}{2} c_2(w) w^2}$ for $w\geq 1.2$, which simplifies~\eqref{eqn:IEI-bound-pf-13} to
  \begin{equation} \label{eqn:IEI-bound-pf-13.5}
  \centering
  \begin{aligned}
        \Phi\bigl( -w+ & \phi(c_1(w) w) \bigr) &< c_{\pi 1} e^{-\frac{1}{2} w^2}. 
  \end{aligned}
\end{equation}
where $c_{\pi 1}=\frac{1+\pi}{2\pi}$.
Therefore, by~\eqref{eqn:IEI-bound-pf-12} and~\eqref{eqn:IEI-bound-pf-13.5}, if $w\geq 1.2$,
  \begin{equation} \label{eqn:IEI-bound-pf-14}
  \centering
  \begin{aligned}
       \Pbb\left\{I_t(\xbm)< -\sigma_{t}(\xbm) w +EI_{t}(\xbm)\right\} < c_{\pi 1}  e^{-\frac{1}{2} w^2}.
  \end{aligned}
\end{equation}
Combining~\eqref{eqn:IEI-bound-pf-14} with~\eqref{eqn:IEI-bound-pf-4} and~\eqref{eqn:IEI-bound-pf-5}, we have
  \begin{equation} \label{eqn:IEI-bound-pf-15}
  \centering
  \begin{aligned}
       \Pbb\left\{ \left|I_t^f(\xbm) - EI_{t}^f(\xbm)\right| > w \sigma_{t}^f(\xbm) \right\} <c_{\alpha} e^{-\frac{1}{2} w^2},
  \end{aligned}
\end{equation}
  where $c_{\alpha}=\frac{1+2\pi}{2\pi}$ for $w\geq 1.2$.
  The probability in~\eqref{eqn:IEI-bound-pf-15} monotonically decreases with $w$.
  Let
    $\delta = c_{\alpha} e^{-\frac{1}{2} w^2}$.
  Then, taking the logarithm of $\delta$ leads to
    $\log( \frac{1+2\pi}{2\pi \delta}) = \frac{1}{2} w^2$.
  Let $\beta=\max\{w^2,1.2^2\}$, and the proof is complete.
\end{proof}

The relationship between $I_t^f(\xbm)$ and $EI_t^f(\xbm)$ under the GP prior assumption is given in the following lemma.
\begin{lemma}\label{lem:IEIbound-ratio-bay}
Given $\delta\in(0,1)$, let $\beta = 2\log(2c_{\alpha}/\delta)$, where $c_{\alpha} = \frac{1+2\pi}{2\pi}$. At given $\xbm\in C$ and $t\in \Nbb$,   
\begin{equation} \label{eqn:IEI-bound-ratio-bay-1}
 \centering
  \begin{aligned}
    \frac{\tau(-\sqrt{\beta})}{\tau(\sqrt{\beta})} I_t^f(\xbm) \leq  EI_{t}^f(\xbm),
  \end{aligned}
\end{equation}
  holds with probability $\geq 1-\delta$
\end{lemma}
\begin{proof}
   From Lemma~\ref{lem:fmu}, with probability $\geq 1-\delta$,~\eqref{eqn:fmu-1} stands. 
   If $f^+_{t}-f(\xbm) \leq 0$, then $I_t(\xbm) = 0$. Since $EI_t(\xbm)\geq 0$,~\eqref{eqn:IEI-bound-ratio-bay-1} is trivial.
   If  $f^+_{t}-f(\xbm) > 0$, by Lemma~\ref{lem:fmu}, 
 \begin{equation} \label{eqn:IEI-bound-ratio-bay-pf-1}
 \centering
  \begin{aligned}
     f^+_{t}-\mu_{t}^f(\xbm) &=      f^+_{t}-f(\xbm)+f(\xbm)-\mu_{t}^f(\xbm) >f(\xbm)-\mu_{t}^f(\xbm)\\
               &>-\sqrt{\beta}\sigma_{t}^f(\xbm),
  \end{aligned}
\end{equation}
  with probability greater than $\geq 1-\delta/2$. From the monotonicity of $\tau$, we have 
 \begin{equation} \label{eqn:IEI-bound-ratio-bay-pf-1.5}
 \centering
  \begin{aligned}
     \tau\left(\frac{f^+_{t}-\mu_{t}^f(\xbm)}{\sigma_{t}^f(\xbm)}\right) >\tau(-\sqrt{\beta}),
  \end{aligned}
\end{equation}
 and therefore,  
 \begin{equation} \label{eqn:IEI-bound-ratio-bay-pf-2}
 \centering
  \begin{aligned}
     EI_t^f(\xbm) = \sigma_t^f(\xbm)\tau\left(\frac{f^+_{t}-\mu_{t}^f(\xbm)}{\sigma_{t}^f(\xbm)}\right) >\tau(-\sqrt{\beta})\sigma_{t}^f(\xbm),
  \end{aligned}
\end{equation}
  with probability greater than $ 1-\delta/2$.
  Using $\delta/2$ in Lemma~\ref{lem:IEIbound}, 
  \begin{equation} \label{eqn:IEI-bound-ratio-bay-pf-4}
  \centering
  \begin{aligned}
      I_t^f(\xbm)- EI_{t}^f(\xbm) \leq \sqrt{\beta}\sigma_{t}^f(\xbm),  
  \end{aligned}
\end{equation}
  with probability $\geq 1-\delta/2$.
  Applying~\eqref{eqn:IEI-bound-ratio-bay-pf-4} to~\eqref{eqn:IEI-bound-ratio-bay-pf-2} with union bound leads to 
 \begin{equation} \label{eqn:IEI-bound-ratio-bay-pf-5}
 \centering
  \begin{aligned}
     EI_t^f(\xbm)  > \frac{\tau(-\sqrt{\beta})}{\sqrt{\beta}+\tau(-\sqrt{\beta})} I_t^f(\xbm) = \frac{\tau(-\sqrt{\beta})}{\tau(\sqrt{\beta})} I_t^f(\xbm),
  \end{aligned}
\end{equation}
  with probability greater than $1-\delta$.
\end{proof}

We can now prove Theorem~\ref{theorem:CEI-convg-noise}.
\begin{proof}
  From Lemma~\ref{lem:f-bound-bayesian},
  \begin{equation} \label{eqn:CEI-convg-pf-3}
  \centering
  \begin{aligned}
        \sum_{i=0}^{t-1}  f^+_{i}-f^+_{i+1} = f^+_0 - f_t^+ \leq 2 M_f,
  \end{aligned}
  \end{equation}
  with probability $\geq 1-\delta/3$. 
  Next, consider $f_i^+-\mu_i^f(x_{i+1})$. 
  Recall that $\beta_t = 2\log(3\pi_t/\delta)$. From Lemma~\ref{lem:fmu-t} Lemma~\ref{lem:sigma-noise} and Lemma~\ref{lem:variancebound}, we have 
  \begin{equation} \label{eqn:CEI-convg-pf-new-1}
  \centering
  \begin{aligned}
        \sum_{i=0}^{t-1}  \max\{f_i^+-\mu_i^f(\xbm_{i+1}),0 \}=& \sum_{i=0}^{t-1}\max\{f^+_{i}-f(\xbm_{i+1})+f(\xbm_{i+1})-\mu_{i}^f(\xbm_{i+1}),0\}\\
          \leq& \sum_{i=0}^{t-1} f^+_{i}-f^+_{i+1}+\beta^{1/2}_t\sigma_{i}^f(\xbm_{i+1})\leq 2M_f +\beta^{1/2}_t\sqrt{C_{\gamma}t\gamma_t^f},
  \end{aligned}
  \end{equation}
  with probability $\geq 1-2\delta/3$ via union bound. 
  Given that $ \max\{f_i^+-\mu_i^f(\xbm_{i+1}),0 \}\geq 0$, $ \max\{f_i^+-\mu_i^f(\xbm_{i+1}),0 \} \geq \frac{2M_f}{k}+\frac{\beta^{1/2}_t}{k}\sqrt{C_{\gamma}t\gamma_t^f}$ at most $k$ times for any $k\in \Nbb$ with probability $\geq 1-2\delta/3$.  
  Choose $k=[t/2]$, where $[x]$ is the largest integer smaller than $x$ so that $2k\leq t\leq 2(k+1)$.
  Then, choose the first index $t_k$ where $k \leq t_k \leq 2k$ so that $\max\{f_{t_k}^+-\mu_{t_k}^f(\xbm_{t_k+1}),0 \} < \frac{2M_f}{k}+\frac{\beta^{1/2}_t}{k}\sqrt{C_{\gamma}t\gamma_t^f}$. As discussed above, such a $t_k$ exists with probability $\geq 1-2\delta/3$. 
  We note that maximum information gain and its upper bound does not depend on the optimization path. 
  Importantly, the choice of $t_k$ does not depend on random information after iteration $t_k$.

  From Lemma~\ref{lem:IEIbound-ratio-bay} and $\beta=2\log(6c_{\alpha}/\delta)$, with probability $\geq 1-\delta/3$,
  \begin{equation} \label{eqn:CEI-convg-pf-4}
  \centering
  \begin{aligned}
          r_t =& f^+_{t}-f(\xbm^*)\leq f^+_{t_k}-f(\xbm^*)\leq I_{t_k}^f(\xbm^*)\\
               \leq& \frac{\tau(\beta^{1/2})}{\tau(-\beta^{1/2})} EI_{t_k}^f(\xbm^*) 
               = c_{\tau}(\beta)\frac{P_{t_k}(\xbm^*)}{P_{t_k}(\xbm^*)} EI_{t_k}^f(\xbm^*) 
               \leq c_{\tau}(\beta)\frac{P_{t_k}(\xbm_{t_k+1})}{P_{t_k}(\xbm^*)} EI_{t_k}^f(\xbm_{t_k+1}) \\
                =& c_{\tau}(\beta)\frac{P_{t_k}(\xbm_{t_k+1})}{P_{t_k}(\xbm^*)}  \left[(f^+_{t_k}-\mu_{t_k}^f(\xbm_{t_k+1}))\Phi(z_{t_k}^f(\xbm_{t_k+1}))+\sigma_{t_k}^f(\xbm_{t_k+1})\phi(z_{t_k}^f(\xbm_{t_k+1}))\right]\\
            \leq& c_{\tau}(\beta) \frac{P_{t_k}(\xbm_{t_k+1})}{P_{t_k}(\xbm^*)}\left[(f^+_{t_k}-\mu_{t_k}^f(\xbm_{t_k+1}))\Phi(z_{t_k}^f(\xbm_{t_k+1}))+0.4\sigma_{t_k}^f(\xbm_{t_k+1})\right] ,
   \end{aligned}
  \end{equation}
   where the last inequality uses $\phi(\cdot)<0.4$.
%    From Lemma~\ref{lem:fmu}, we have with probability $\geq 1-2\delta/3$, 
%   \begin{equation} \label{eqn:CEI-convg-pf-5}
%  \centering
%  \begin{aligned}
%          f^+_{t_k}-\mu_{t_k}^f(\xbm_{t_k+1}) =& f^+_{t_k}-f(\xbm_{t_k+1})+f(\xbm_{t_k+1})-\mu_{t_k}^f(\xbm_{t_k+1}) \\
%          \leq& f^+_{t_k}-f^+_{t_k+1}+\beta^{1/2}\sigma_{t_k}^f(\xbm_{t_k+1})\leq \frac{2M_f}{k} %+\beta^{1/2}\sigma_{t_k}^f(\xbm_{t_k+1}).
%    \end{aligned}
%  \end{equation}
From the choice of $t_k$,~\eqref{eqn:CEI-convg-pf-4} leads to  
  \begin{equation} \label{eqn:CEI-convg-pf-6}
  \centering
  \begin{aligned}
          r_t  \leq& c_{\tau}(\beta)\frac{P_{t_k}(\xbm_{t_k+1})}{P_{t_k}(\xbm^*)} \left[\frac{2M_f}{k}+\frac{\beta^{1/2}_t}{k}\sqrt{C_{\gamma}t\gamma_t^f}+(0.4+\beta^{1/2})\sigma_{t_k}^f(\xbm_{t_k+1})\right],\\
  \end{aligned}
  \end{equation}
  with probability $\geq 1-\delta$. 
  Next, we consider the probability function $P_{t_k}$ at $\xbm^*$ and $\xbm_{t_k+1}$.
   Using the fact that $c(\xbm^*) \leq 0$, we have by Lemma~\ref{lem:fmu} at $\xbm^*$ and $t_k$,
   \begin{equation} \label{eqn:CEI-convg-pf-7}
  \centering
  \begin{aligned}
     \mu_{t_k}^c(\xbm^*) \leq B_c \sigma_{t_k}^c(\xbm^*)+ c(\xbm^*) \leq  B_c \sigma_{t_k}^c(\xbm^*).
    \end{aligned}
  \end{equation}
  Thus,  we can write  
  \begin{equation} \label{eqn:CEI-convg-pf-8}
  \centering
  \begin{aligned}
     \frac{-\mu_{t_k}^c(\xbm^*)}{\sigma_{t_k}^c(\xbm^*)} \geq -B_c.
    \end{aligned}
  \end{equation}
  From the monotonicity of $\Phi$, we have
   \begin{equation} \label{eqn:CEI-convg-pf-9}
  \centering
  \begin{aligned}
     \Phi\left(\frac{-\mu_{t_k}^c(\xbm^*)}{\sigma_{t_k}^c(\xbm^*)}\right) \geq \Phi\left(-B_c\right),
    \end{aligned}
  \end{equation}
   Using~\eqref{eqn:CEI-convg-pf-9}, the $P_{t_k}$ functions have  
   \begin{equation} \label{eqn:CEI-convg-pf-10}
  \centering
  \begin{aligned}
     \frac{P_{t_k}(\xbm_{t_k+1})}{P_{t_k}(\xbm^*)}  \leq  \frac{1}{\Phi(-B_c)}.
    \end{aligned}
  \end{equation} 
  Applying~\eqref{eqn:CEI-convg-pf-10} to~\eqref{eqn:CEI-convg-pf-6}, we have
    \begin{equation} \label{eqn:CEI-convg-pf-11}
  \centering
  \begin{aligned}
          r_t  \leq&  c_{\tau}(\beta)\frac{1}{\Phi(
          -B_c)}\left[\frac{2M_f}{k}+\frac{\beta^{1/2}_t}{k}\sqrt{C_{\gamma}t\gamma_t^f}+(0.4+\beta^{1/2})\sigma_{t_k}^f(\xbm_{t_k+1})\right],\\
  \end{aligned}
  \end{equation}
  with probability $\geq 1-\delta$.

\end{proof}
The proof of Theorem~\ref{prop:EI-convg-rate-gamma} is next.
\begin{proof}
  From Lemma~\ref{lem:EI-sigma-gamma},
    $\sigma_{i}^f(\xbm_{i+1})  \geq \frac{\sqrt{\gamma_t^f t}}{k}$, where $i=0,\dots,t-1$, at most $k$ times for any $k\in \Nbb$ and $k\leq t$.
   Choose $k=[t/3]$ which leads to $3k\leq t\leq 3(k+1)$. 
   Let $t_k$ be the first index in $[k,3k]$ so that $\sigma_{t_k}^f(\xbm_{t_k+1}) \leq \frac{\sqrt{t \gamma^f_t}}{k}$ and $\max\{f_{t_k}^+-\mu_{t_k}^f(\xbm_{t_k+1}),0 \} < \frac{2M_f}{k}+\frac{\beta^{1/2}_t}{k}\sqrt{C_{\gamma}t\gamma_t^f}$, which exists with probablity $\geq 1-2\delta/3$. Notice that $\beta_t^{1/2}=\mathcal{O}(\log^{1/2}(t))$. Following the proof for~\eqref{eqn:CEI-convg-pf-11}, we have
    \begin{equation} \label{eqn:CEI-convg-pf-12}
  \centering
  \begin{aligned}
          r_t  \leq&  c_{\tau}(\beta)\frac{1}{\Phi(-B_c)}\left[\frac{2M_f}{k}+\frac{\beta^{1/2}_t}{k}\sqrt{C_{\gamma}t\gamma_t^f}+(0.4+\beta^{1/2})\frac{\sqrt{t \gamma^f_t}}{k}\right],\\
  \end{aligned}
  \end{equation}
  with probability $\geq 1-\delta$.
  Using Lemma~\ref{lem:gammarate},  the proof is complete.

\end{proof}

\section{Test problems}\label{se:exp_contour}

The mathematical formulations of the testing problems in Section~\ref{se:testproblem} are given in this section.
The objective, constraint functions and the optimal $f$ of Problem 1 is given below.
\begin{equation} \label{ex3}
  \centering
  \begin{aligned}
    &f(\xbm)=  \sin(x_1)+ x_2,\\  
	&c(\xbm)= \sin(x_1)\sin(x_2) +0.95\leq 0,\\
					&x_i \in [0, 6],  i= 1, 2,\\
					&f^* = 0.25.
\end{aligned}
  \end{equation}
 
The objective, constraint functions and the optimal $f$ of Problem 2 is given below.
\begin{equation} \label{ex4}
  \centering
  \begin{aligned}
    &f(\mathbf{x})= x_1+ x_2\\
					&c_{1}(\xbm)=  -0.5 \sin( 2\pi (x_1^2-2 x_2))-x_1-2x_2 +1.5\leq 0\\
					&c_{2}(\xbm)= x_1^2+x^2_2 -1.5\leq 0\\
					&x_i \in [0, 1], i = 1, 2\\
					&f^* = 0.6.
\end{aligned}
  \end{equation}

  The objective, constraint functions and the optimal $f$ of Problem 3 is given below.
\begin{equation} \label{ex5}
  \centering
  \begin{aligned}
    &f(\mathbf{x})=  x_1 + x_2 + x_3 + x_4\\
					&c_{1}=  1.1-\sum_{i=1}^4 E_i \exp\left( \sum_{j=1}^4 -A_{j,i}(x_j - P_{j,i})^2 \right)\\
					&x_i \in [0, 1], i = 1,\dots,4\\
					&E = [1, 1.2, 3, 3.2]^\top\\
					&P = 
					\begin{bmatrix}
						0.131 & 0.232 & 0.234 & 0.404 \\
						0.169 & 0.413 & 0.145 & 0.882 \\
						0.556 & 0.830 & 0.352 & 0.873 \\
						0.012 & 0.373 & 0.288 & 0.574
					\end{bmatrix}\\
					&A = 
					\begin{bmatrix}
						10 & 0.05 & 3 & 17 \\
						3 & 10 & 3.5 & 8 \\
						17 & 17 & 1.7 & 0.05 \\
						3.5 & 0.1 & 10 & 10
					\end{bmatrix}\\
					&f^* = 0.
\end{aligned}
  \end{equation}

The objective, constraint functions and the optimal $f$ of Problem 4 is given below.
\begin{equation} \label{ex6}
  \centering
  \begin{aligned}
    &f(\mathbf{x}) = - \sum_{i=1}^{4} \alpha_i \exp \left( -\sum_{j=1}^{6} A_{ij} (x_j - P_{ij})^2 \right)\\
					&c(\xbm)= \sum_{j=1}^4 x_j -3 \\
					&x_i \in [0, 1], i = 1, \dots,6\\
					&\alpha = [1.0, 1.2, 3.0, 3.2]^\top\\
					&A = 
					\begin{bmatrix}
						10 & 3.0 & 17 & 3.5 & 1.7 & 8.0 \\
						0.05 & 10 & 17 & 0.1 & 8.0 & 14 \\
						3.0 & 3.5 & 1.7 & 10 & 17 & 8.0 \\
						17 & 8.0 & 0.05 & 10 & 0.1 & 14
					\end{bmatrix}\\
					&P = 
					\begin{bmatrix}
						0.131 & 0.170 & 0.557 & 0.012 & 0.828 & 0.587 \\
						0.233 & 0.414 & 0.831 & 0.374 & 0.100 & 0.999 \\
						0.235 & 0.145 & 0.352 & 0.288 & 0.305 & 0.665 \\
						0.405 & 0.883 & 0.873 & 0.574 & 0.109 & 0.038
					\end{bmatrix}\\
					&f^* = -3.32.
\end{aligned}
  \end{equation}

% \begin{figure} 
%   \centering
%   \includegraphics[width=0.9\textwidth]{img/"ex7_contour".png}
%   \caption{Contour plots for the objective function (left) and constraint function (right) for example 1. 
%             The feasible region is marked on the plots, along with the global optimum ($*$ sign).}
% \label{fig:exrkhs}
% \end{figure}

% \begin{figure} 
%   \centering
%   \includegraphics[width=0.9\textwidth]{img/"ex6_contour".png}
%   \caption{Contour plots for the objective function (left) and constraint function (right) for example 2. 
%             The feasible region is marked on the plots, along with the global optimum ($*$ sign).}
% \label{fig:exgp}
% \end{figure}

\begin{figure}
  \centering
  \includegraphics[width=0.9\textwidth]{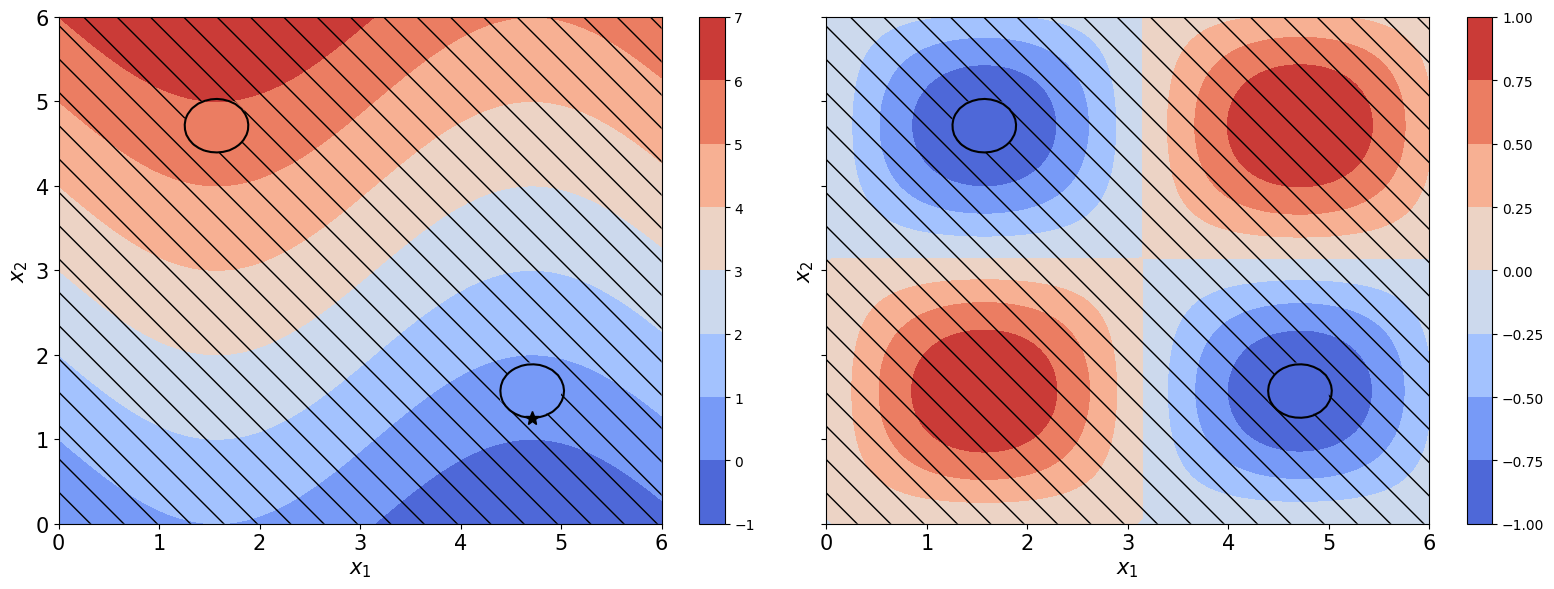}
  \caption{Contour plots for the objective function (left) and constraint function (right) for Problem 1. 
            The infeasible region is marked on the plots. The global optimum is marked with a star sign.}
\label{fig:ex1_contour}
\end{figure}
\begin{figure}
  \centering
  \includegraphics[width=0.9\textwidth]{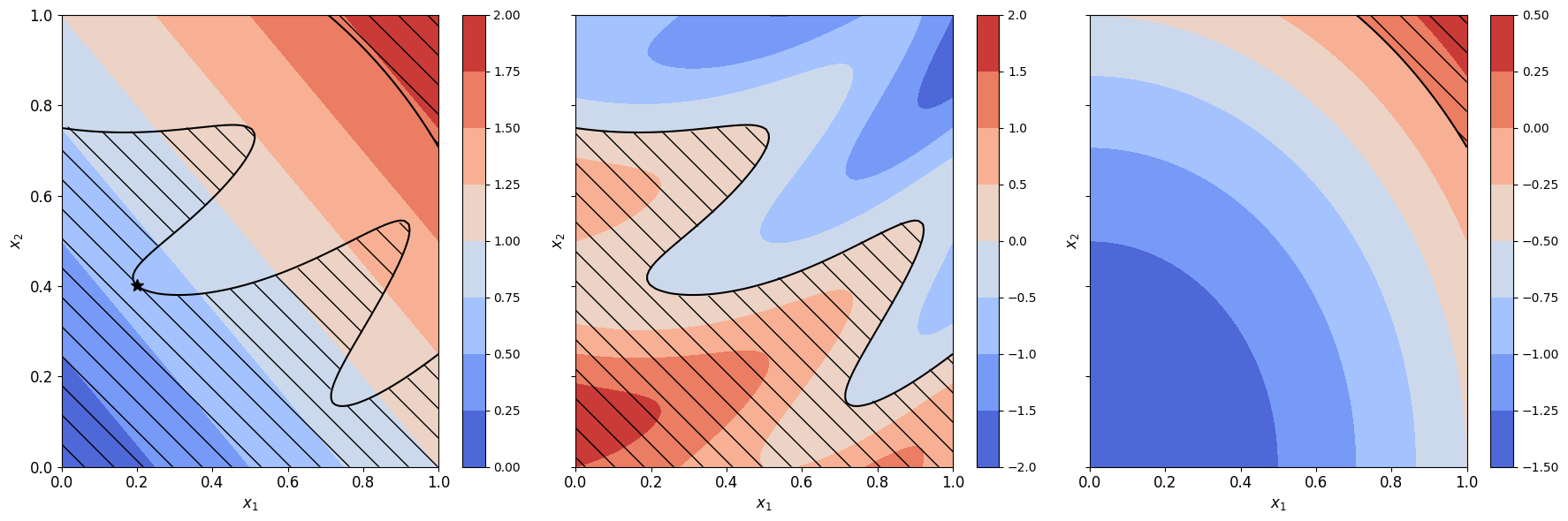}
  \caption{Contour plots for the objective function (left) and the two constraint functions (middle and right) for Problem 2. 
The infeasible region is marked with black line on the objective contour. The global optimum is marked with a star sign.}
\label{fig:ex2_contour}
\end{figure}
\begin{figure}
    \centering
    \includegraphics[width=0.9\textwidth]{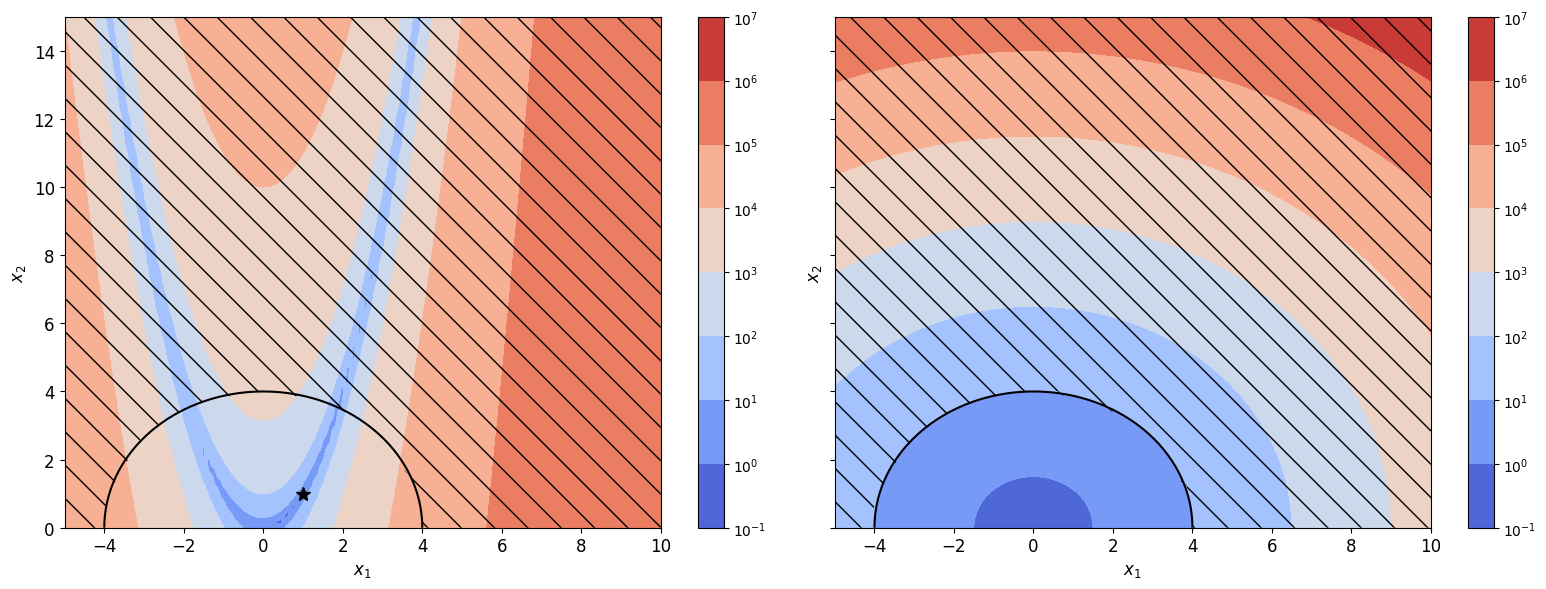}
  \caption{Contour plots for the objective function (left) and constraint function (right) for Problem 5. 
            The infeasible region is marked on the plots. The global optimum is marked with a star sign.}
  \label{fig:ex5_contour}
  \end{figure}

The objective, constraint functions and the optimal $f$ of Problem 5 is given below.
\begin{equation} \label{ex7}
  \centering
  \begin{aligned}
       &f(\mathbf{x})= 100 (x_2 - x_1^2)^2 + (1 - x_1)^2\\
					&c_{1}(\xbm)=  \sqrt{x_1^2+x_1^2}-4\\
					&c_{2}(\xbm)= x_1^2+x^2_2 -1.5,\\
					&x_1 \in [-5, 10], \ x_2 \in [0,15]\\
					&f^* = 0.
\end{aligned}
  \end{equation}

The contour plots of Problem 1, 2, and 5 are given in Figure~\ref{fig:ex1_contour},~\ref{fig:ex2_contour}, and~\ref{fig:ex5_contour}.
%Proof of Lemma~\ref{lem:sqregretbound}.
%\begin{proof}
% By Lemma~\ref{lem:discrete-bound},~\ref{lem:cplusregret} and~\ref{lem:variancebound}, with probability greater than $1-\delta$, 
% \begin{equation} \label{eqn:sqrreg-3}
%  \centering
%  \begin{aligned}
%    \sum_{t=1}^T  r_t^2\leq& 4 \alpha_T \sum_{t=1}^T(\sigma_{t-1}^c(x_t))^2 
%             \leq C_1\alpha_T \gamma_T,\\
%  \end{aligned}
%\end{equation}
%where $C_1 = 8 /log(1+\sigma^{-2})$.
%\end{proof}
%Proof of Theorem~\ref{thm:disc-regretbound}
%\begin{proof}
%   From Lemma~\ref{lem:sqregretbound} and the Cauchy-Schwarz inequality, we have with probability $\geq 1-\delta$
% \begin{equation} \label{eqn:regbound-2}
 %  \centering
 %  \begin{aligned}
 %         R_T^2 = \left(\sum_{t=1}^T r_t\right)^2 \leq T \sum_{t=1}^T r_t^2 \leq C_1 T\alpha_T \gamma^c_T.
  % \end{aligned}
%\end{equation}
%  Therefore,~\eqref{eqn:regbound-1} holds with probability $\geq 1-\delta$. 
%\end{proof}

%%%%%%%%%%%%%%%%%%%%%%%%%%%%%%%%%%%%%%%%%%%%%%%%%%%%%%%%%%%%

\end{document}